\documentclass[final]{siamart220329}

\usepackage{lipsum}
\usepackage{amsfonts}
\usepackage{graphicx}
\usepackage{epstopdf}
\usepackage{algorithmic}
\ifpdf
  \DeclareGraphicsExtensions{.eps,.pdf,.png,.jpg}
\else
  \DeclareGraphicsExtensions{.eps}
\fi

\usepackage{amssymb}

\usepackage{amsmath,a4wide,mathrsfs,color,array}
\usepackage[utf8]{inputenc}
\usepackage[T1]{fontenc}
\usepackage{subcaption}
\usepackage{graphicx}
\usepackage{float}
\usepackage{multirow}
\usepackage{hyperref}
\usepackage{cleveref}
\usepackage{tikz-cd}
\usepackage{fancyvrb}
\usepackage{upgreek}
\usepackage{xcolor}
\usepackage{listings}

\usetikzlibrary{shapes,arrows,positioning}

\usepackage{tikz}
\usetikzlibrary{shapes,arrows,positioning}
\usetikzlibrary{babel}
\usetikzlibrary{shapes,arrows}
\definecolor{dark_blue_pers}{RGB}{46,87,144}
\definecolor{blue_pers}{RGB}{54,104,171}
\definecolor{grey_pers}{RGB}{245,245,245}
\definecolor{red_pers}{RGB}{213,78,33}
\usetikzlibrary{backgrounds, scopes}

\usepackage{color}
\definecolor{dark_blue}{RGB}{46,87,144}
\definecolor{dark_green}{RGB}{0,100,0}

\newcommand{\argmin}{\operatorname{argmin}}
\newcommand{\nc}{\newcommand}

\nc{\norm}[2]{\left\|#1\right\|_{#2}}

\usepackage{algorithm}
\usepackage{algorithmic}

\usepackage{xcolor}

\definecolor{dark_blue}{RGB}{46,87,144}
\newcommand\rev{}

\nc{\IC}{\mathbb{C}}
\nc{\IE}{\mathbb{E}}
\nc{\IN}{\mathbb{N}}
\nc{\IP}{\mathbb{P}}
\nc{\IR}{\mathbb{R}}

\nc{\be}{\begin{equation}}
\nc{\ee}{\end{equation}}

\nc{\fA}{\mathsf{A}}
\nc{\fL}{\mathsf{L}}

\nc{\eps}{\varepsilon}

\nc{\mA}{\mathcal{A}}
\nc{\mB}{\mathcal{B}}
\nc{\mE}{\mathcal{E}}
\nc{\mG}{\mathcal{G}}
\nc{\mL}{\mathcal{L}}
\nc{\mD}{\mathcal{D}}
\nc{\mM}{\mathcal{M}}
\nc{\mN}{\mathcal{N}}
\nc{\mP}{\mathcal{P}}
\nc{\mR}{\mathcal{R}}
\nc{\mT}{\mathcal{T}}

\nc{\bx}{{\bf x}}
\nc{\bn}{{\bf n}}
\nc{\bq}{{\bf q}}
\nc{\bW}{{\bf W}}
\nc{\bA}{{\bf A}}
\nc{\bX}{{\bf X}}

\nc{\bb}{{\bf b}}
\nc{\bz}{{\bf z}}
\nc{\by}{{\bf y}}
\nc{\bu}{{\bf u}}
\nc{\btheta}{{\boldsymbol{\theta}}}
\nc{\bxi}{{\boldsymbol{\xi}}}

\nc{\tk}{{\textup{k}}}
\nc{\tC}{{\textrm{C}}}

\nc{\tB}{{\textrm{B}}}
\nc{\tD}{{\textrm{D}}}
\nc{\tV}{{\textrm{V}}}
\nc{\tU}{{\textrm{U}}}
\nc{\tG}{{\textrm{H}}}

\definecolor{codegreen}{rgb}{0,0.6,0}
\definecolor{codegray}{rgb}{0.5,0.5,0.5}
\definecolor{codepurple}{rgb}{0.58,0,0.82}
\definecolor{backcolour}{rgb}{0.95,0.95,0.92}

\lstdefinestyle{mystyle}{
    backgroundcolor=\color{backcolour},   
    commentstyle=\color{codegreen},
    keywordstyle=\color{magenta},
    numberstyle=\tiny\color{codegray},
    stringstyle=\color{codepurple},
    basicstyle=\ttfamily\footnotesize,
    breakatwhitespace=false,         
    breaklines=true,                 
    captionpos=b,                    
    keepspaces=true,                 
    numbers=left,                    
    numbersep=5pt,                  
    showspaces=false,                
    showstringspaces=false,
    showtabs=false,                  
    tabsize=2
}

\nc{\loc}{{_\textup{loc}}}

\usepackage[normalem]{ulem}

\newsiamremark{remark}{Remark}

\headers{PINNs for operator equations with stochastic data}{P. Escapil-Inchauspé, and G. A. Ruz}

\title{Physics-informed neural networks for operator equations with stochastic data\thanks{Submitted to the editors \date.
\funding{This work was funded by the FES-UAI postdoc grant, ANID Postdoctorado 3230088, ANID PIA/BASAL FB0002, and ANID/PIA/ANILLOS ACT210096.}}}

\author{Paul Escapil-Inchauspé\thanks{Facultad de Ingenier\'ia y Ciencias, Universidad Adolfo Ib\'a\~nez, Santiago, Chile (\email{paul.escapil@edu.uai.cl, gonzalo.ruz@uai.cl}), and Data Observatory Foundation, Santiago, Chile.}
\and Gonzalo A. Ruz\footnotemark[2]
\thanks{Center of Applied Ecology and Sustainability (CAPES), Santiago, Chile.}
}

\usepackage{amsopn}

\ifpdf
\hypersetup{
  pdftitle={Physics-informed neural networks for operator equations with stochastic data},
  pdfauthor={P. Escapil-Inchauspé, and G. A. Ruz}
}
\fi

\begin{document}

\maketitle

\begin{abstract}
We consider the computation of statistical moments to operator equations with stochastic data. We remark that application of PINNs---referred to as TPINNs---allows to solve the induced tensor operator equations under minimal changes of existing PINNs code\rev{, and enabling handling of non-linear and time-dependent operators}. We propose two types of architectures, referred to as vanilla and multi-output TPINNs, and investigate their benefits and limitations. Exhaustive numerical experiments are performed; demonstrating applicability and performance; raising a variety of new promising research avenues.
\end{abstract}



\section{Introduction}\label{sec:intro}Uncertainty quantification (UQ) is paramount in domains ranging from \rev{aerospace} exploration to electronic design automation. Uncertain data or source term delivers an abstract operator equation with stochastic data of the form:
\be\label{eq:stochasticData}
\fA u (\omega) = f (\omega)\quad \IP\text{-a.e. }\omega \in \Omega.
\ee
We aim at computing the statistical moments for $u(\omega)$, as being for any integer $\tk \geq 1:$
$$
\mM^\tk[u] : = \int_\Omega u(\bx_1,\omega) \cdots u(\bx_k,\omega) d \IP,
$$
the latter amounting \rev{to solving} the following tensor operator equation:
$$
(\fA \otimes \cdots \otimes \fA) \mM^\tk[u]   =\mM^\tk[f].
$$
Tensor operator equations are prone to the infamous curse of dimensionality \cite{vonPetersdorff2006}. To circumvent this limitation, coupling: (i) numerical scheme such as finite or boundary element methods \cite{steinbach2007numerical}; with (ii) sparse tensor approximation \cite{gerstner1998numerical} is common in literature. We refer to the reference work of von Petersdorff and Schwab \cite{vonPetersdorff2006} and applications \cite{sparse3,escapil2020helmholtz,multigroup}.

\rev{However}, the aforementioned method:
\begin{enumerate}
\item Restricts to strongly elliptic linear operators \cite[Section 2]{vonPetersdorff2006}\rev{;}
\item Can be non-trivial to implement, as numerical solvers are not adapted to high\rev{-}order tensors. Indeed, they commonly only deliver matrix-product operations\rev{;}
\item Leads to optimal yet possibly slow convergence for iterative solvers \cite[Section 6.2]{escapil2020helmholtz}.
\end{enumerate}
Concerning Item 1., non-linear operators are amenable to linear tensor equations with stochastic data under additional requirements \cite{chernov2013first}. Regarding Item 2., notice that for matrices $\bA , \bX \in \IC^{N,N}$ and integer $N\geq 1$, there holds that \cite{MIKAMatrixFreeKL}:
\be 
(\bA \otimes \bA) \textbf{vec}(\bX) = \bA \bX\bA^T,
\ee
where $\textbf{vec}(\bX)$ stacks the columns of $\bX$ one below the others, and $\bA^T$ is the transpose of $\bA$. This allows to solve tensor operator equations for $\tk=2$ provided a matrix-matrix class. We refer the reader to \cite[Section 2.2]{MIKAMatrixFreeKL} for higher $\tk$ and sparse matrices. To the authors\rev{'} knowledge, numerical experiments in \rev{the} literature are restricted to $\tk=2$, despite enjoying a complete theory \cite{schwab2003sparseHigh}.

Recently, physics-informed neural networks (PINNs) were introduced in \cite{RAISSI2019686}. Praised for their versatility, they apply to forward and inverse problems involving partial differential equations (PDEs). They are constructed over deep neural networks \cite{bengio2017deep}, inheriting their capability to approximate high-dimensional and non-linear mappings \cite{schwab2019deep,scarabosio2021deep}. Amongst \rev{other} applications of PINNs, we mention inverse problems \cite{Chen2020}, inverse design \cite{luluhard2021}, and fractional operators \cite{fPINNs}. Recently, Mishra and Molinaro \cite{Mishra2020EstimatesOT,Mishra2021Inverse} proposed a general theory to quantify the generalization error for PINNs.

All throughout, we restrict to UQ for deterministic PDEs with a random load in \eqref{eq:stochasticData}. For the sake of completeness, we mention UQ for parametric stochastic operators \cite{HOQMC} (e.g.~elliptic operators \cite{barth2011multi}) and for stochastic PDEs via PINNs \cite{zhangSPDEPINNs}. \rev{Furthermore,} deep neural networks (resp.~PINNs) were used to model uncertain surrogates \cite{scarabosio2021deep,DeepUQ} (resp.~\cite{ZHU201956}). We also put forth Bayesian PINNs \cite{yang2021b}, PINNs for determining the total uncertainty \cite{ZHANG2019108850}, and GAN PINNs \cite{yang2019adversarial}. A complete overview of these \rev{methods} and their practical implementation \rev{is} available in NeuralUQ\footnote{\url{https://github.com/Crunch-UQ4MI/neuraluq}} \cite{zou2022neuraluq}. Lastly, we mention multi-output (MO)-PINNs for UQ \cite{yang2022multi}.

In this work, we apply PINNs to tensor operator equations, referred to as Tensor PINNs (TPINNs). These novel PINNs inherits the following interesting properties:
\begin{itemize}\setlength{\itemsep}{0pt}
\item[$\checkmark$] They scale well with increasing $\tk$;
\item[$\checkmark$] They require minimal changes---specifically, adapt the data generation---to existing PINNs code;
\item[$\checkmark$] They allow to consider time-dependent and non-linear problems;
\item[$\checkmark$] They can be extended to---ill-posed---inverse problems \cite{Chen2020}.
\end{itemize}

Moreover, PINNs solve operator equations in strong form. This allows for a rather simple expression for the Kronecker product (as opposed to operators equations in weak form \cite{vonPetersdorff2006} or matrices). Furthermore, operator equations can often be described in (higher) \emph{mixed Sobolev regularity spaces} \cite{vonPetersdorff2006}, this extra natural smoothness supporting the use of strong form---though being out of the scope of this work.

Inspired by the framework in \cite{Mishra2020EstimatesOT}, we provide a bound for the generalization error of vanilla (V)-TPINNs in \Cref{thm:boundGeneralization}. This result allows to understand better TPINNs and supplies a strong theoretical background to their formulation.

Alongside, we remark that TPINNs involve higher order differentiation as $\tk$ increases (refer also to \cite{zhu2021local,gladstone2022fo}). Accordingly, we introduce a MO variant of V-TPINNs, referred to as MO-TPINNs, and applicable for the scalar case. It consists \rev{of} using successive applications of the operator as variables, bounding the order of differentiation for the loss with $\tk$. This results in lowering the computational cost at the expense of an increase of hyper-parameters---the $\tk$ terms in the loss function.

Numerical experiments demonstrate the practical simplicity of TPINNs and their prominent performance. They allow to obtain a surprisingly accurate approximation for few collocation points $N$ and training epochs, e.g,~a $L^2$ relative error of $\sim 5\%$ for $N\sim 10^3$ collocation points after $\sim 10^4$ epochs for $\tk=2$ and a two dimensional Helmholtz equation (see later on in \Cref{subsec:Helmholtz}).

We provide a fair comparison of V-TPINNs and MO-TPINNs, and study the trade-off between training performance and computational requirements. We consider separable and Gaussian covariance kernels as right-hand side. This work is structured as follows: We formulate the framework for operator equations with stochastic data in \Cref{sec:Opeq}, we introduce TPINNs in \Cref{sec:TensorPINNs} and supply convergence analysis in \Cref{sec:Convergence}. Next, we discuss implementation in \Cref{sec:Implementation}, we present numerical experiments in \Cref{sec:NumExp} and conclude in \Cref{sec:Conclusion}.

\section{Operator equations with stochastic data}\label{sec:Opeq}To begin with, we set the notations that will be used all throughout this manuscript.
\subsection{General Notation}\label{subsec:GeneralNotation}For a natural number $\tk$, we set $\IN_\tk : = \{\tk,\tk+1,\cdots \}$. Let $D \subseteq \IR^d$ for $d \in \IN_1$ be an open set. For $p>0$, $L^p(D)$ is the standard class of functions with bounded $L^p$-norm over $D$. Given $s\in \IR$, $q\geq 0$, $p \in [1,\infty]$, we refer to \cite{steinbach2007numerical} for the definitions of Sobolev function spaces $H^s(D)$. Norms are denoted by $\|\cdot\|$ with subscripts indicating the associated functional space. For $\tk\in\IN_1$ and $\bx_i \in \IR^d$, $i=1,\cdots,\tk$, we set $\bx := (\bx_1, \cdots ,\bx_\tk)$. $\tk$-fold tensors quantities are denoted with parenthesized \rev{super}scripts, e.g.~$f^{(\tk)} : = f \otimes \cdots \otimes f$. Their diagonal part is referred as:
\be\label{eq:diagonal}
\text{diag}(f^{(\tk)}) = f^{(\tk)}|_{\bx_1 = \cdots = \bx_\tk}= f(\bx_1) \otimes \cdots\otimes f(\bx_1).
\ee
For $X,Y$ separable Hilbert spaces,
we set $\fA \in \mB(X,Y)$ the space of bounded mappings from $X$ to $Y$ and define the unique bounded tensor product operator \cite{omran2016some}:
$$
\fA^{(\tk)}: = \fA \otimes \cdots \otimes \fA \in \mB( X^{(\tk)}, Y^{(\tk)}).
$$
Notice that if $\fA$ is continuous, $\mB(X,Y)\equiv \mL(X,Y)$, the latter being the space of continuous linear operators.

\subsection{Abstract problem}\label{subsec:AbstractProblem}
All throughout this manuscript, let $(\Omega, \mA, \IP)$ be a probability space, $X,Y$ separable Hilbert spaces, and $\tk\in \IN_1$. For $u:\Omega \to X$ a random field in the Bochner space $L^\tk(\Omega,\IP;X)$ \cite{vonPetersdorff2006}, we introduce the statistical moments:
\be\label{eq:defmoment}
\mM^\tk [ u(\omega) ] : = \int_\Omega u(\bx_1, \omega) \cdots u(\bx_\tk, \omega) d\IP(\omega),
\ee
with $\mM^1=\IE$ being the expectation.

The abstract problem reads: Given $\fA :X \to Y$ and $f \in L^\tk(\Omega,\IP;Y)$, we seek $u \in L^\tk(\Omega,\IP;X)$ such that:
\be \label{eq:stochastic}
\fA u(\omega) = b(\omega) \quad \text{for } \IP\text{-a.e. }\omega \in \Omega.
\ee
Application of \eqref{eq:defmoment} to \eqref{eq:stochastic} yields the following operator equation: Given $\tC^\tk := \mM^\tk[f] \in Y^{(\tk)}$, we seek $\Sigma^\tk := \mM^\tk[u]\in X^{(\tk)}$ such that
\be\label{eq:opeq} 
\fA^{(\tk)}\Sigma^\tk = \tC^\tk .
\ee
In the sequel, we assume that \eqref{eq:opeq} has a unique solution. Furthermore, we suppose that $\fA$ admits a linearization $\fL \in \mL(X,Y)$ with bounded inverse $\fL^{-1} \in \mL(Y,X)$, i.e.
\be\label{eq:linearization}
\fA u - \fA v = \fL (u -v) \quad \forall \, u,v\in X
\ee
with 
$$
\|\fL^{-1}\|_{Y \to X} \leq \gamma_\fL^{-1}  <\infty
$$
wherein $\gamma_\fL> 0$. This setting corresponds to Example 2 in \cite[Section 2.1]{Mishra2020EstimatesOT}. Thus, one has that for any $u,v \in X$:
\be\label{eq:continuousdep1}
\|u-v\|_X =\| \fL^{-1} (\fA u-\fA v)\|_X\leq \gamma_\fL^{-1} \|\fA u -\fA v\|_Y.
\ee
As a consequence, application of the statistical moments to \eqref{eq:linearization} for $u,v\in L^{(\tk)}(\Omega,\IP;X)$ yields a linearization for $\fA^{(\tk)}$ as:
$$
\fA^{(\tk)}\tU-\fA^{(\tk)}\tV = \fL^{(\tk)}(\tU - \tV),
$$
wherein $\tU := \mM^{(\tk)}[u]$ and $\tV := \mM^{(\tk)}[v]$. Finally, there holds that:
\be\label{eq:continuousdep2}
\| \tU - \tV \|_{X^{(\tk)}}  = \left\| (\fL^{-1})^{(\tk)} \left(\fA^{(\tk)} \tU - \fA^{(\tk)} \tV \right) \right\|_{X^{(\tk)}} \leq \frac{1}{\gamma_{\fL}^\tk} \left\| \fA^{(k)}\tU - \fA^{(\tk)}\tV\right\|_{Y^{(\tk)}}.
\ee
Stability bound \eqref{eq:continuousdep2} will be key in yielding a bounded generalization error for TPINNs in \Cref{sec:Convergence}. 

Generally, operator equations are used to describe boundary value problems \cite{vonPetersdorff2006} and can be recast as:
\be \label{eq:decomposition}
\fA = (\fA_D,\fA_B)
\ee
with $\fA_D$ and $\fA_B$ defined over vector spaces with values in $D$ and $\partial D$ respectively---$\partial D=\Gamma$ or $\Gamma \times (t=0)$ for time-dependent problem. This observation will be key in simplifying TPINNs later on in \Cref{sec:TensorPINNs}.

\section{TPINNs}\label{sec:TensorPINNs}The proposed method consists in solving operator equation \eqref{eq:opeq} by means of PINNs. For the sake of simplicity, we restrict \rev{ourselves} to $Y=L^2(D;\IR^m)$. Notice that we could apply our framework to UQ for inverse problems (refer e.g.~to \cite{Mishra2021Inverse}). Let $\sigma$ be a smooth activation function. \rev{Throughout,} we consider an input $\bx = (\bx_1,\cdots ,\bx_\tk)\in D^{(\tk)}$. Following \cite[Section 2.1]{lu2021deepxde} and notations in \cite{escapil2020helmholtz}, we define $\mN\mN$ as being a $L$-layer neural network with $\mN_l$ neurons in the $l$-th layer for $1 \leq l \leq L-1$ ($\mN_0= d \tk$ and $\mN_L$ to be determined later on). For $1 \leq l \leq L$, let us denote the weight matrix and bias vector in the $l$-th layer by $\bW^l \in \IR^{\mN_l \times \mN_{l-1}}$ and $\bb^l \in \IR^{\mN_l}$, respectively. A mapping $\bx\mapsto \bz(\bx)$ can be approximated by a deep (feedforward) neural network defined as follows:
\be\label{eq:TPINNs} 
\begin{array}{rll}
\text{input layer:} \quad & \bx \in (\IR^d)^{(\tk)},\\
\text{hidden layers:} \quad & \bz^l (\bx) = \sigma( \bW^l \bz^{l-1} (\bx) + \bb^l) \in \IR^{\mN_l} \quad \text{for} \quad 1 \leq l \leq L-1,\\
\text{output layer:} \quad & \bz^L(\bx) = \bW^L\bz^{L-1} (\bx) + \bb^L  \in \IR^{\mN_L}. 
\end{array}
\ee

Finally, we introduce the collocation points $\mT := \left\{ \bx_i |~ \bx_i\in D^{(\tk)}\right\}_{i=1}^N$ of cardinality $N \in \IN_1$. PINNs are commonly optimized via ADAM \cite{kingma2014adam} with a given learning rate $l_r$ over a fixed number of $\text{epochs}$. Derivatives are evaluated by means of automatic differentiation (AD) \cite{lu2021deepxde}, allowing \rev{us} to consider general pseudo-differential operators. Further application of L-BFGS \cite{lbfgs} can improve training \cite{lu2021deepxde}. 

In this work, we make an extensive use of hard boundary conditions (BCs) \cite{luluhard2021,lagaris1998artificial,lu2021deepxde} in order to restrict the collocation points to $D^{(\tk)}$ and to \rev{greatly simplify} TPINNs (refer to \Cref{rmk:hardBCs} for more details). Hard BCs consist in applying a transformation $\Sigma^\tk\mapsto \hat{\Sigma}^\tk$ in such a way that $\hat{\Sigma}^\tk$ fulfills the BCs.

\begin{remark}[Hard BCs]\label{rmk:hardBCs}Using a transformation to enforce hard BCs allows to greatly simplify the formulation for TPINNs by reducing the tensor operator equation over $D^{\tk}$. For instance, let us consider the Laplace operator with Dirichlet BCs $\gamma_D u \mapsto u|_\Gamma$. The formulation for $\tk=2$ is
$$
\begin{cases}
(-\Delta \otimes -\Delta ) \mM^2[u]& = \mM^2[f_D] \quad \text{in}\quad D\times D,\\
(-\Delta  \otimes \gamma_{D} )\mM^2[u] &= \IE[f_Dg] \quad\text{in } \quad D \times \Gamma,\\
(\gamma_{D} \otimes -\Delta  ) \mM^2[u] &= \IE[gf_D] \quad \text{in } \quad \Gamma \times D,\\
(\gamma_{D} \otimes \gamma_{D}) \mM^2[u]& = \mM^2[g]\quad  \text{on } \quad \Gamma \times \Gamma.
\end{cases}
$$
In this case, the tensor equation for $\tk$ has $2^\tk$ terms, which would originate a $2^\tk$-terms loss function for V-TPINNs in \eqref{eq:Loss}.
\end{remark}

One generally surveys that the model generalizes well by predicting the solution over a $N^\text{test}$-cardinality training set $\mT^\text{test}$ defined again over $D^{(\tk)}$. For TPINNs with constant width, we set $\mN = \mN_j$, $j=1,\cdots,L-1$.

\subsection{V-TPINNs}\label{subsec:VTPINNs}V-TPINNs correspond to setting $N_L=m$ in \eqref{eq:TPINNs} and applying \eqref{eq:TPINNs} to \eqref{eq:opeq}, delivering an approximation
\be\label{eq:solVTPINNs} 
\bz^L=\Sigma_\theta^\tk.
\ee
The residual for V-TPINNs is:
\be \label{eq:residualVTPINNs}
\bxi_\theta : = \fA^{(\tk)}\Sigma^\tk_\theta - \tC^\tk
\ee
and the loss function reads (for a MC points distribution):
\be\label{eq:Loss}
\mL_\theta :=\frac{1}{N}\sum_{\bx \in \mT}  \bxi_\theta^2 .
\ee
We seek at obtaining:
$$
\theta^\star := \argmin_{\theta \in \Theta} \mL_\theta 
$$
with $\mL_\theta$ in \eqref{eq:Loss} yielding
\be\label{eq:bestOpt}
\Sigma_\star^\tk := \Sigma_{\theta^\star}^\tk.
\ee 
When hard BCs are used, notice that the residual \eqref{eq:residualVTPINNs} reduces to 
$$\bxi_\theta : = \fA_D^{(\tk)}\Sigma^\tk_\theta - \tC_D^\tk\quad \text{with}\quad \Sigma^\tk_\theta \equiv \Sigma^\tk_{D,\theta} \quad \text{and} \quad \tC_D = \mM^{\tk}[f|_D].$$

An example of V-TPINN is showcased in \Cref{fig:tensorPINN}. It is paramount to \rev{note} that the loss function has only one term, but that the latter involves the evaluation of $\fA^{(\tk)}$. Also, $\tk$ only impacts the input layer, affecting moderately---linearly---the size of the neural network. 

\tikzset{every picture/.style={line width=0.75pt}} 
\begin{figure}[htb!]
\centering
\resizebox{10.5cm}{!} {
\input{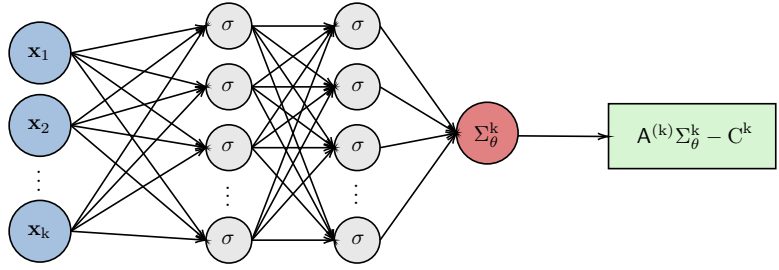}
}
\caption{Schematic representation of a V-TPINN. The neural networks \rev{have} $L-1=2$ hidden layers.}
\label{fig:tensorPINN}
\end{figure}

\subsection{MO-TPINNs}\label{subssec:MOTPINNs}As will be noticed later in numerical experiments, V-TPINNs can become impractical for increasing $\tk$ when $\fA$ involves derivatives. For instance, set $\fA_D=-\Delta$ (refer e.g.~to the example in \Cref{rmk:hardBCs}). The Laplace operator involves a second-order derivative; hence the TPINN loss involves $2 \tk$-order derivatives. \rev{However,} it is known that AD scales poorly with high-order derivatives (refer to \cite{zhu2021local,gladstone2022fo}). To be precise, the computational requirements of AD grow exponentially with the order of differentiation \cite{bettencourt2019taylor}. In \cite{bettencourt2019taylor}, the authors propose a method to lower this cost for Jax \cite{jax2018github} backend. Though, the latter does not allow to evaluate mixed derivatives, as occurs in V-TPINNs.

To remedy this concern, we introduce a MO architecture, inspired by local Galerkin methods \cite{zhu2021local} and \cite{gladstone2022fo}. 

Let us assume that $\fA = (\fA_D,\fA_B)$ in \eqref{eq:decomposition} and $N_L = m=1$ in \eqref{eq:TPINNs} i.e.~$Y=L^2(D;\IR)$. There holds that:
$$
\fA_D \otimes \cdots \otimes \fA_D = \fA_D \circ \cdots \circ \fA_D.
$$
We set:
\be\label{eq:MOPINNssol}
\begin{cases}
\fA_D \tV_1 & = \tC_D^\tk, \\
\fA_D \tV_2 &= \tV_1 ,\\
 &\vdots \\
\fA_D \rev{\Sigma^\tk} & = \rev{\tV_{\tk-1}}.
\end{cases}
\ee 
MO-TPINNs amount to use hard BCs and \rev{apply} \eqref{eq:TPINNs} to $(\tV_{1,\theta},\cdots, \tV_{\tk,\theta})$ in \eqref{eq:MOPINNssol}, hence the approximate:
\be 
\bz^L =  (\tV_{\rev{1,}\theta},\cdots, \tV_{\tk,\theta}).
\ee
Having introduced the residuals
\be \label{eq:residualsMOTPINNs}
\begin{cases}
\bxi_1 & = \fA\rev{_D} \tV_1  - \rev{\tC^\tk_D}, \\
 \bxi_2 &=  \fA\rev{_D} \tV_2 - \tV_1, \\
 &\vdots \\
\bxi_{\tk} & =\rev{\fA_D \Sigma^\tk - \tV_{\tk-1}},
\end{cases}
\ee 
the composite loss function is defined for any $\upomega_j >0$, $j=1,\cdots,\tk$ as:
\be
\mL_\theta =  \sum_{j=1}^\tk \upomega_j \mL_\theta^j
\ee
with 
\be\label{eq:LossMO}
\mL_\theta^j =  \frac{1}{N}\sum_{\bx \in \mT}  \bxi_{j,\theta}(\bx)^2.
\ee
Acknowledge that the residual in \eqref{eq:residualsMOTPINNs} does not involve powers of $\fA$, though inducing an output and loss function with $\tk$ terms. In \rev{summary}, we represent a MO-TPINN in \Cref{fig:MO-TPINN}. 
\tikzset{every picture/.style={line width=0.75pt}} 
\begin{figure}[htb!]
\centering
\resizebox{10.5cm}{!} {
\input{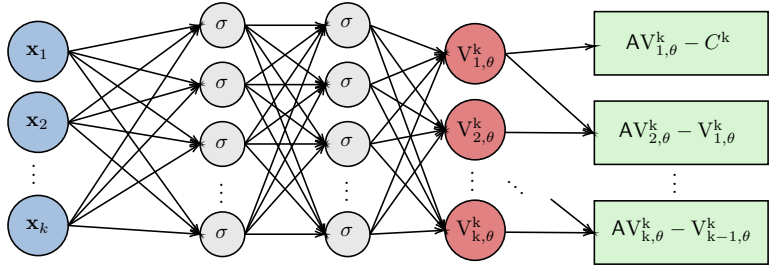}
}
\caption{Schematic representation of a MO-TPINN. The neural networks has $L-1=2$ hidden layers.}
\label{fig:MO-TPINN}
\end{figure}

\section{Convergence analysis for V-TPINNs}\label{sec:Convergence}In this section, we apply the proceeding of \cite{Mishra2020EstimatesOT} to V-TPINNs. We restrict to V-TPINNs, \rev{since} stability bounds for MO-TPINNs are more involved and remain elusive. The residual PINNs loss in \eqref{eq:Loss} induces approximating integrals for $\tG \in Y^{(\tk)} = L^2(D;\IR^m)^{(\tk)}$:
$$
\hat{\tG} = \int_{D^{(\tk)}} \tG (\bx_1, \cdots ,\bx_\tk) \ d \bx_1 \cdots d \bx_\tk.
$$
We assume that we are given a quadrature to approximate $\hat{\tG}$ as:
\be 
\hat{\tG}_N : = \sum_{i=1}^N w_i \tG ( \by^i)
\ee
for weights $w_i$ and quadrature points $\by^i \in D^{(\tk)}$. We further assume that the quadrature error is bounded as:
\be \label{eq:boundQuadrature}
| \hat{\tG}_N - \hat{\tG} | \leq c_\text{quad} (\|\tG\|_{Y^{(\tk)}}) N^{-\alpha}
\ee
for some $\alpha > 0$. Concerning $\alpha$, Monte-Carlo sampling methods yield $\alpha=1/2$ in the root mean square and are a method of choice for very high dimensions \cite{Mishra2020EstimatesOT}. Notice that we defined the loss in this sense in \Cref{eq:Loss} and \eqref{eq:LossMO}. For moderately high dimensions, we can use Quasi Monte-Carlo sampling methods, whose performance relies on the smoothness of $Y$.

We seek at quantifying the total error, also referred to as \emph{generalization error}: 
\be\label{eq:genErrorDef}
\eps_G = \eps_G(\theta^\star) := \|\Sigma^\tk_{\star} - \Sigma^\tk \|_{X^{(\tk)}}
\ee
with $\Sigma^\tk_\star$ and $\Sigma^\tk_\theta$ in \eqref{eq:solVTPINNs} and \eqref{eq:bestOpt} respectively. 

The \emph{training error} is given by:
\be\label{eq:TrainingError}
\eps_T (\theta)^2 = \sum_{i=1}^N w_i \bxi_\star (\by_i)^2 \approx \|\bxi_\star\|_{Y^{(\tk)}}.
\ee
Acknowledge that \eqref{eq:genErrorDef} measures the total error in the domain space, while \eqref{eq:TrainingError} is for the residual in range space. We aim at understanding how these errors relate with each other.

In \Cref{fig:ConvergenceBounds}, we depict the error for V-TPINNs. Remark that the generalization error is bounded from below by the best approximation error of the neural net class (the class of functions that the neural net can approximate). 

\tikzset{every picture/.style={line width=0.75pt}} 
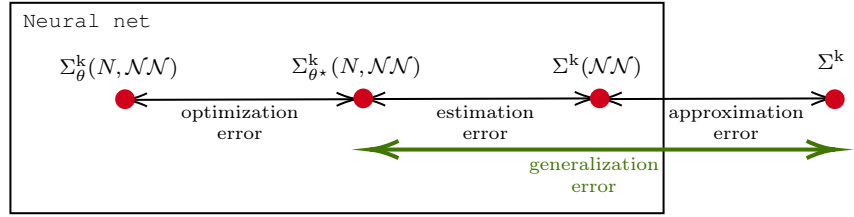
\begin{figure}[htb!]
\centering
\resizebox{11.5cm}{!} {
\tikzset{every picture/.style={line width=0.75pt}} 

\begin{tikzpicture}[x=0.75pt,y=0.75pt,yscale=-1,xscale=1]

\draw    (139.3,54.29) -- (251.4,53.96) ;
\draw [shift={(253.4,53.95)}, rotate = 179.83] [color={rgb, 255:red, 0; green, 0; blue, 0 }  ][line width=0.75]    (10.93,-3.29) .. controls (6.95,-1.4) and (3.31,-0.3) .. (0,0) .. controls (3.31,0.3) and (6.95,1.4) .. (10.93,3.29)   ;
\draw [shift={(137.3,54.3)}, rotate = 359.83] [color={rgb, 255:red, 0; green, 0; blue, 0 }  ][line width=0.75]    (10.93,-3.29) .. controls (6.95,-1.4) and (3.31,-0.3) .. (0,0) .. controls (3.31,0.3) and (6.95,1.4) .. (10.93,3.29)   ;
\draw  [draw opacity=0][fill={rgb, 255:red, 208; green, 2; blue, 27 }  ,fill opacity=1 ] (132.4,53.95) .. controls (132.4,51.24) and (134.59,49.05) .. (137.3,49.05) .. controls (140.01,49.05) and (142.2,51.24) .. (142.2,53.95) .. controls (142.2,56.66) and (140.01,58.85) .. (137.3,58.85) .. controls (134.59,58.85) and (132.4,56.66) .. (132.4,53.95) -- cycle ;
\draw    (257.7,53.89) -- (369.8,53.56) ;
\draw [shift={(371.8,53.55)}, rotate = 179.83] [color={rgb, 255:red, 0; green, 0; blue, 0 }  ][line width=0.75]    (10.93,-3.29) .. controls (6.95,-1.4) and (3.31,-0.3) .. (0,0) .. controls (3.31,0.3) and (6.95,1.4) .. (10.93,3.29)   ;
\draw [shift={(255.7,53.9)}, rotate = 359.83] [color={rgb, 255:red, 0; green, 0; blue, 0 }  ][line width=0.75]    (10.93,-3.29) .. controls (6.95,-1.4) and (3.31,-0.3) .. (0,0) .. controls (3.31,0.3) and (6.95,1.4) .. (10.93,3.29)   ;
\draw  [draw opacity=0][fill={rgb, 255:red, 208; green, 2; blue, 27 }  ,fill opacity=1 ] (250.8,53.55) .. controls (250.8,50.84) and (252.99,48.65) .. (255.7,48.65) .. controls (258.41,48.65) and (260.6,50.84) .. (260.6,53.55) .. controls (260.6,56.26) and (258.41,58.45) .. (255.7,58.45) .. controls (252.99,58.45) and (250.8,56.26) .. (250.8,53.55) -- cycle ;
\draw    (375.3,53.89) -- (487.4,53.56) ;
\draw [shift={(489.4,53.55)}, rotate = 179.83] [color={rgb, 255:red, 0; green, 0; blue, 0 }  ][line width=0.75]    (10.93,-3.29) .. controls (6.95,-1.4) and (3.31,-0.3) .. (0,0) .. controls (3.31,0.3) and (6.95,1.4) .. (10.93,3.29)   ;
\draw [shift={(373.3,53.9)}, rotate = 359.83] [color={rgb, 255:red, 0; green, 0; blue, 0 }  ][line width=0.75]    (10.93,-3.29) .. controls (6.95,-1.4) and (3.31,-0.3) .. (0,0) .. controls (3.31,0.3) and (6.95,1.4) .. (10.93,3.29)   ;
\draw  [draw opacity=0][fill={rgb, 255:red, 208; green, 2; blue, 27 }  ,fill opacity=1 ] (368.4,53.55) .. controls (368.4,50.84) and (370.59,48.65) .. (373.3,48.65) .. controls (376.01,48.65) and (378.2,50.84) .. (378.2,53.55) .. controls (378.2,56.26) and (376.01,58.45) .. (373.3,58.45) .. controls (370.59,58.45) and (368.4,56.26) .. (368.4,53.55) -- cycle ;
\draw  [draw opacity=0][fill={rgb, 255:red, 208; green, 2; blue, 27 }  ,fill opacity=1 ] (485.2,53.95) .. controls (485.2,51.24) and (487.39,49.05) .. (490.1,49.05) .. controls (492.81,49.05) and (495,51.24) .. (495,53.95) .. controls (495,56.66) and (492.81,58.85) .. (490.1,58.85) .. controls (487.39,58.85) and (485.2,56.66) .. (485.2,53.95) -- cycle ;
\draw   (80.33,6) -- (405,6) -- (405,110.67) -- (80.33,110.67) -- cycle ;
\draw [color={rgb, 255:red, 65; green, 117; blue, 5 }  ,draw opacity=1 ][line width=1.5]    (262,79) -- (487.33,79) ;
\draw [shift={(490.33,79)}, rotate = 180] [color={rgb, 255:red, 65; green, 117; blue, 5 }  ,draw opacity=1 ][line width=1.5]    (14.21,-4.28) .. controls (9.04,-1.82) and (4.3,-0.39) .. (0,0) .. controls (4.3,0.39) and (9.04,1.82) .. (14.21,4.28)   ;
\draw [shift={(259,79)}, rotate = 0] [color={rgb, 255:red, 65; green, 117; blue, 5 }  ,draw opacity=1 ][line width=1.5]    (14.21,-4.28) .. controls (9.04,-1.82) and (4.3,-0.39) .. (0,0) .. controls (4.3,0.39) and (9.04,1.82) .. (14.21,4.28)   ;

\draw (103.2,30.4) node [anchor=north west][inner sep=0.75pt]  [font=\footnotesize]  {$\Sigma _{\theta }^{\mathrm{k}}( N,\mathcal{NN})$};
\draw (218.6,30) node [anchor=north west][inner sep=0.75pt]  [font=\footnotesize]  {$\Sigma _{\theta ^{\star }}^{\mathrm{k}}( N,\mathcal{NN})$};
\draw (348.2,30) node [anchor=north west][inner sep=0.75pt]  [font=\footnotesize]  {$\Sigma ^{\mathrm{k}}(\mathcal{NN})$};
\draw (480,28) node [anchor=north west][inner sep=0.75pt]  [font=\footnotesize]  {$\Sigma \mathrm{^{k}}$};
\draw (85.14,9.72) node [anchor=north west][inner sep=0.75pt]  [font=\footnotesize] [align=left] {{\fontfamily{pcr}\selectfont Neural net}};
\draw (194.44,65.33) node  [font=\scriptsize] [align=left] {\begin{minipage}[lt]{57.89pt}\setlength\topsep{0pt}
\begin{center}
optimization error
\end{center}

\end{minipage}};
\draw (316.44,65.33) node  [font=\scriptsize] [align=left] {\begin{minipage}[lt]{52.33pt}\setlength\topsep{0pt}
\begin{center}
estimation error
\end{center}

\end{minipage}};
\draw (441.11,65.33) node  [font=\scriptsize] [align=left] {\begin{minipage}[lt]{64.63pt}\setlength\topsep{0pt}
\begin{center}
approximation error
\end{center}

\end{minipage}};
\draw (370.44,91.33) node  [font=\scriptsize,color={rgb, 255:red, 65; green, 117; blue, 5 }  ,opacity=1 ] [align=left] {\begin{minipage}[lt]{64.24pt}\setlength\topsep{0pt}
\begin{center}
generalization error
\end{center}

\end{minipage}};

\end{tikzpicture}}
\caption{Illustration of the total error for TPINNs. In \Cref{thm:boundGeneralization} we provide a bound for the generalization error for V-TPINNs.}
\label{fig:ConvergenceBounds}
\end{figure}

We are ready to state a fundamental result of this work, namely a bound for the generalization error of V-TPINNs.

\begin{theorem}[Generalization error for V-TPINNs]\label{thm:boundGeneralization}Under the present setting, the generalization error in \eqref{eq:genErrorDef} is bounded as:
\be\label{eq:boundGeneralization}
\eps_G\leq  \gamma_\fL^{-\tk} \eps_T + \gamma_\fL^{-\tk} c_\textup{quad}^{1/2}  N^{- \alpha / 2},
\ee 
with $\gamma_\fL$ the stability constant in \eqref{eq:continuousdep1} and $c_\textup{quad}$ in \eqref{eq:boundQuadrature}.
\end{theorem}
\begin{proof}Under the present setting, remark that:
\begin{align*} 
\eps_G^2 &= \|\Sigma^\tk_\star - \Sigma^\tk\|_{X^{(\tk)}}^2 \quad \text{by definition \eqref{eq:genErrorDef},}\\
&\leq \gamma_\fL^{-2\tk} \|\fA^{(\tk)} \Sigma_\star^\tk - \fA^{(\tk)}\Sigma^\tk\|^2_{L^2(D)^{(\tk)}}\quad \text{by application of \eqref{eq:continuousdep2} to } \Sigma_\star=\tU , \Sigma^\tk =\tV, \\
& = \gamma_\fL^{-2\tk} \| \bxi_\star \|_{Y^{(\tk)}} \quad \text{by \eqref{eq:residualVTPINNs},}\\
&\leq  \gamma_\fL^{-2\tk} \left(\rev{\eps_T^2} + c_\text{quad} N^{-\alpha} \right)\quad \text{by \eqref{eq:boundQuadrature},}
\end{align*}
yielding the desired result.
\end{proof}
\Cref{thm:boundGeneralization} states that the generalization is bounded by: (i) one term depending on the training error (tractable) and (ii) another one depending on the number of collocation points. This shows that for enough collocation points, one has a controlled generalization error depending on the training error modulo the stability constant for the problem under consideration. Also, this result proves that the novel V-TPINNs has the same capabilities as PINNs, which will be verified in \Cref{sec:NumExp}.

\section{Implementation}\label{sec:Implementation}
We detail the implementation of TPINNs in DeepXDE \cite{lu2021deepxde}. Readers are referred to the preceding reference for the steps of general PINNs implementation. In Listing \ref{lst:tensor}, we detail the code that was added to handle random sampling for tensor points $\bx \in D^{(\tk)}$. These few lines allows to sample random points in $D^{(\tk)}$. Notice than one could handle soft BCs similarly by adding a method for random boundary points.
~\\
\begin{lstlisting}[language=Python, caption=Tensor Geometry class in DeepXDE that allows to sample random points in $D^{(\tk)}$., label=lst:tensor]
class GeometryXGeometry:
    def __init__(self, geometry, k):
        self.geometry = geometry
        self.dim = geometry.dim ** k
        self.k = k
    
    def random_points(self, n, random="pseudo"):
        xx = []
        for i in range(self.k):
            xx.append(self.geometry.random_points(n, random=random))
        return np.hstack(xx)
\end{lstlisting}
For the sake of clarity, we revisit the flowchart of PINNs according to DeepXDE \cite[Figure 5]{lu2021deepxde} in \Cref{fig:flowchartTPINNs}. We remark that TPINNs can be applied with minor changes to existing code (namely, the implementation of a sampler), and without having to access to the solver core. Specific loss and BCs are enforced via the main code.

\tikzset{every picture/.style={line width=0.75pt}} 
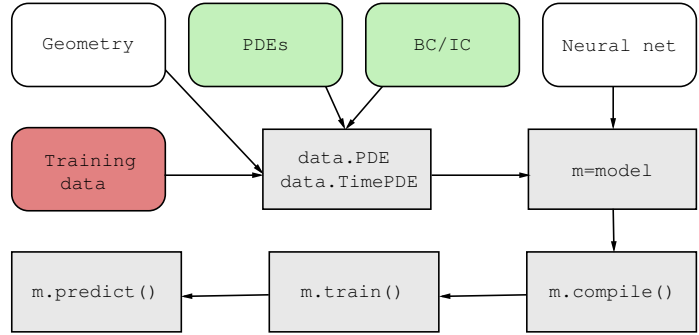
\begin{figure}[htb!]
\centering
\resizebox{9.5cm}{!} {
\tikzset{every picture/.style={line width=0.75pt}} 
\begin{tikzpicture}[x=0.75pt,y=0.75pt,yscale=-1,xscale=1]

\draw   (106,52.4) .. controls (106,46.66) and (110.66,42) .. (116.4,42) -- (191.6,42) .. controls (197.34,42) and (202,46.66) .. (202,52.4) -- (202,83.6) .. controls (202,89.34) and (197.34,94) .. (191.6,94) -- (116.4,94) .. controls (110.66,94) and (106,89.34) .. (106,83.6) -- cycle ;
\draw  [fill={rgb, 255:red, 195; green, 241; blue, 187 }  ,fill opacity=1 ] (216.67,52.4) .. controls (216.67,46.66) and (221.33,42) .. (227.07,42) -- (302.27,42) .. controls (308.01,42) and (312.67,46.66) .. (312.67,52.4) -- (312.67,83.6) .. controls (312.67,89.34) and (308.01,94) .. (302.27,94) -- (227.07,94) .. controls (221.33,94) and (216.67,89.34) .. (216.67,83.6) -- cycle ;
\draw  [fill={rgb, 255:red, 195; green, 241; blue, 187 }  ,fill opacity=1 ] (327.34,52.4) .. controls (327.34,46.66) and (332,42) .. (337.74,42) -- (412.94,42) .. controls (418.68,42) and (423.34,46.66) .. (423.34,52.4) -- (423.34,83.6) .. controls (423.34,89.34) and (418.68,94) .. (412.94,94) -- (337.74,94) .. controls (332,94) and (327.34,89.34) .. (327.34,83.6) -- cycle ;
\draw   (438,52.4) .. controls (438,46.66) and (442.66,42) .. (448.4,42) -- (523.6,42) .. controls (529.34,42) and (534,46.66) .. (534,52.4) -- (534,83.6) .. controls (534,89.34) and (529.34,94) .. (523.6,94) -- (448.4,94) .. controls (442.66,94) and (438,89.34) .. (438,83.6) -- cycle ;
\draw  [fill={rgb, 255:red, 226; green, 129; blue, 129 }  ,fill opacity=1 ] (106,130.15) .. controls (106,124.41) and (110.66,119.75) .. (116.4,119.75) -- (191.6,119.75) .. controls (197.34,119.75) and (202,124.41) .. (202,130.15) -- (202,161.35) .. controls (202,167.09) and (197.34,171.75) .. (191.6,171.75) -- (116.4,171.75) .. controls (110.66,171.75) and (106,167.09) .. (106,161.35) -- cycle ;
\draw  [fill={rgb, 255:red, 155; green, 155; blue, 155 }  ,fill opacity=0.23 ] (263,120.75) -- (368,120.75) -- (368,170.75) -- (263,170.75) -- cycle ;
\draw  [fill={rgb, 255:red, 155; green, 155; blue, 155 }  ,fill opacity=0.23 ] (429,120.75) -- (534,120.75) -- (534,170.75) -- (429,170.75) -- cycle ;
\draw  [fill={rgb, 255:red, 155; green, 155; blue, 155 }  ,fill opacity=0.23 ] (428,197.5) -- (534,197.5) -- (534,247.5) -- (428,247.5) -- cycle ;
\draw  [fill={rgb, 255:red, 155; green, 155; blue, 155 }  ,fill opacity=0.23 ] (267,197.5) -- (373,197.5) -- (373,247.5) -- (267,247.5) -- cycle ;
\draw  [fill={rgb, 255:red, 155; green, 155; blue, 155 }  ,fill opacity=0.23 ] (106,197.5) -- (212,197.5) -- (212,247.5) -- (106,247.5) -- cycle ;
\draw    (202,83.6) -- (262.13,147.54) ;
\draw [shift={(263.5,149)}, rotate = 226.76] [fill={rgb, 255:red, 0; green, 0; blue, 0 }  ][line width=0.08]  [draw opacity=0] (7.2,-1.8) -- (0,0) -- (7.2,1.8) -- cycle    ;
\draw    (201.5,149.5) -- (261.5,149.5) ;
\draw [shift={(263.5,149.5)}, rotate = 180] [fill={rgb, 255:red, 0; green, 0; blue, 0 }  ][line width=0.08]  [draw opacity=0] (7.2,-1.8) -- (0,0) -- (7.2,1.8) -- cycle    ;
\draw    (302.27,94) -- (313.67,119.18) ;
\draw [shift={(314.5,121)}, rotate = 245.63] [fill={rgb, 255:red, 0; green, 0; blue, 0 }  ][line width=0.08]  [draw opacity=0] (7.2,-1.8) -- (0,0) -- (7.2,1.8) -- cycle    ;
\draw    (337.74,94) -- (315.8,119.48) ;
\draw [shift={(314.5,121)}, rotate = 310.72] [fill={rgb, 255:red, 0; green, 0; blue, 0 }  ][line width=0.08]  [draw opacity=0] (7.2,-1.8) -- (0,0) -- (7.2,1.8) -- cycle    ;
\draw    (368.5,149.5) -- (428.5,149.5) ;
\draw    (482,94) -- (482,119) ;
\draw [shift={(482,121)}, rotate = 270] [fill={rgb, 255:red, 0; green, 0; blue, 0 }  ][line width=0.08]  [draw opacity=0] (7.2,-1.8) -- (0,0) -- (7.2,1.8) -- cycle    ;
\draw [shift={(430.5,149.5)}, rotate = 180] [fill={rgb, 255:red, 0; green, 0; blue, 0 }  ][line width=0.08]  [draw opacity=0] (7.2,-1.8) -- (0,0) -- (7.2,1.8) -- cycle    ;
\draw    (482,170.5) -- (482.46,196) ;
\draw [shift={(482.5,198)}, rotate = 268.96] [fill={rgb, 255:red, 0; green, 0; blue, 0 }  ][line width=0.08]  [draw opacity=0] (7.2,-1.8) -- (0,0) -- (7.2,1.8) -- cycle    ;
\draw    (428,224.33) -- (375,225.3) ;
\draw [shift={(373,225.33)}, rotate = 358.96] [fill={rgb, 255:red, 0; green, 0; blue, 0 }  ][line width=0.08]  [draw opacity=0] (7.2,-1.8) -- (0,0) -- (7.2,1.8) -- cycle    ;
\draw    (266.67,224.33) -- (213.67,225.3) ;
\draw [shift={(211.67,225.33)}, rotate = 358.96] [fill={rgb, 255:red, 0; green, 0; blue, 0 }  ][line width=0.08]  [draw opacity=0] (7.2,-1.8) -- (0,0) -- (7.2,1.8) -- cycle    ;

\draw (473.66,142.75) node  [font=\small] [align=left] {\begin{minipage}[lt]{25.88pt}\setlength\topsep{0pt}
\begin{center}{\fontfamily{pcr}\selectfont m=model}
\end{center}

\end{minipage}};
\draw (313.73,145.75) node  [font=\small] [align=left] {\begin{minipage}[lt]{59.93pt}\setlength\topsep{0pt}
\begin{center}
{\fontfamily{pcr}\selectfont data.PDE}\\{\fontfamily{pcr}\selectfont data.TimePDE}
\end{center}

\end{minipage}};
\draw (319.22,222.5) node  [font=\small] [align=left] {\begin{minipage}[lt]{53.37pt}\setlength\topsep{0pt}
\begin{center}
{\fontfamily{pcr}\selectfont m.train()}
\end{center}

\end{minipage}};
\draw (158.44,222.5) node  [font=\small] [align=left] {\begin{minipage}[lt]{61.6pt}\setlength\topsep{0pt}
\begin{center}
{\fontfamily{pcr}\selectfont m.predict()}
\end{center}

\end{minipage}};
\draw (481.66,222.5) node  [font=\small] [align=left] {\begin{minipage}[lt]{64.74pt}\setlength\topsep{0pt}
\begin{center}
{\fontfamily{pcr}\selectfont m.compile()}
\end{center}

\end{minipage}};
\draw (153.84,68) node  [font=\small] [align=left] {{\fontfamily{pcr}\selectfont Geometry}};
\draw (264.89,68) node  [font=\small] [align=left] {\begin{minipage}[lt]{47.44pt}\setlength\topsep{0pt}
\begin{center}
{\fontfamily{pcr}\selectfont PDEs}
\end{center}

\end{minipage}};
\draw (375.43,68) node  [font=\small] [align=left] {{\fontfamily{pcr}\selectfont BC/IC}};
\draw (486.11,68) node  [font=\small] [align=left] {{\fontfamily{pcr}\selectfont Neural net}};
\draw (150.71,143.75) node  [font=\small] [align=center] {\begin{minipage}[lt]{36.64pt}\setlength\topsep{0pt}
\begin{center}
{\fontfamily{pcr}\selectfont Training}\\{\fontfamily{pcr}\selectfont data}
\end{center}

\end{minipage}};
\end{tikzpicture}
}
\caption{Flowchart of DeepXDE in \cite[Figure 5]{lu2021deepxde} showing that TPINNs can be added to PINNs code with minor changes. The additions to the source code are in red. Boxes in green are adapted to the tensor setting in the main code via proper definition of the loss function and hard constraints.}
\label{fig:flowchartTPINNs}
\end{figure}

\section{Numerical experiments}\label{sec:NumExp}In this section, we apply TPINNs to several operator equations. 
\subsection{Methodology}\label{subsec:methodology}Throughout, we perform simulations with DeepXDE 1.7.1\footnote{\url{https://deepxde.readthedocs.io/}} \cite{lu2021deepxde} in single float precision on a AMAX DL-E48A AMD Rome EPYC server with 8 Quadro RTX 8000 Nvidia GPUs---each one with a 48 GB memory. We use TensorFlow 2.5.0 \cite{tensorflow2015-whitepaper} as a backend. We disable XLA\footnote{\url{https://www.tensorflow.org/xla}} and sample the collocation points randomly. All the results are available on GitHub\footnote{\url{https://github.com/pescap/TensorPINNs}} and fully reproducible, with a global random seed of $1234$. We use Glorot uniform initialization \cite[Chapter 8]{bengio2017deep}. ``Error'' refers to the $L^2$-relative error, and ``Time'' (in seconds) stands for the training time.

For MO-TPINNs, we set the weights in \eqref{eq:LossMO} as:
$$
\upomega_1 = 1\quad \upomega_j = 1000, \quad j= 1,\cdots ,\tk.
$$

For each case, we compare both V-TPINNs and MO-TPINNs. Each subsection is intended to bring a novelty concerning solving operator equations with stochastic data:
\begin{itemize}\setlength{\itemsep}{0pt}
\item \Cref{subsec:Poisson} - 1D Poisson: Higher order moments $\tk= \{1, \cdots 4\}$;
\item \Cref{subsec:StationarySchrodinger} - 1D Schrödinger: Non-linear operator;
\item \Cref{subsec:Helmholtz} - 2D Helmholtz: Two-dimensional problem with oscillatory behavior;
\item \Cref{subsec:Heat} - 1D Heat: Time-domain operator.
\end{itemize}

For the sake of simplicity, we summarize the parameters and hyper-parameters for TPINNs for each case in \Cref{tab:OverviewHP}.
\begin{table}[h!t]
\renewcommand\arraystretch{1.5}
\begin{center}
\footnotesize
\begin{tabular}{
    |>{\centering\arraybackslash}m{2.5cm}
    |>{\centering\arraybackslash}m{1.1cm}
    |>{\centering\arraybackslash}m{0.8cm}
    |>{\centering\arraybackslash}m{0.8cm}
    |>{\centering\arraybackslash}m{1.1cm}
    |>{\centering\arraybackslash}m{1.1cm}
    |>{\centering\arraybackslash}m{1.1cm}
    |>{\centering\arraybackslash}m{.8cm}|
    }
    \hline
\multirow{2}{*}{Case}  & learning rate & width & depth& \multirow{2}{*}{epochs} & \multirow{2}{*}{$N$} & \multirow{2}{*}{$N^\text{test} $}  & \multirow{2}{*}{$\sigma$} \\ 
  &$l_r$&   $\mN$ &$L-1$ &  & & &\\ \hline \hline
 1D Poisson& $10^{-3*}$ & $50$ & $4$  &  $10000$ & $4000$ & $1000$ & $\tanh$ \\ \hline
 1D Schrödinger & $10^{-3}$ & $50$ & $4$ & $10000$ & $4000$ & $1000$ & $\tanh$ \\\hline
 2D Helmholtz & $10^{-2}$ & $350$ & $2$ & $10000^{**}$ & $5000$ & $5000$ & $\sin$ \\ \hline 
 1D Heat & $10^{-3}$ & $20$ & $3$ & $15000$ & $4000$ & $4000$ & $\tanh$\\ \hline
 \end{tabular}
\end{center}
\caption{Overview of the parameters and hyper-parameters for each case; *$10^{-4}$ for MO-TPINNs; **$15000$ for covariance kernels in \Cref{subsubsec:HelmholtzExp}.}
\label{tab:OverviewHP} 
\end{table}

\subsection{Poisson 1D}\label{subsec:Poisson}To begin with, we consider the Laplace equation in $D := [0,1]$. For $\tk \in \IN_1$, we seek $u \in L^\tk (\Omega,\IP, X)$ such that
$$
- \Delta u (\omega) = f_D(w) \quad \text{and} \quad u(\omega)|_\Gamma = 0\quad \IP\text{-a.e. }\omega \in \Omega
$$
with 
$$
f_D(\omega) = x e^{\mu (\omega)} ,\quad \overline{f}_D:=x, \quad \omega \in \Omega, \quad \mu > 0.
$$
This problem can be cast as \eqref{eq:stochastic} with $\fA :X=H^1(D)\to Y=L^2(D)$ \cite[Remark 3.2]{Mishra2021Inverse}. Acknowledge that for any $\omega \in \Omega$:
\be\label{eq:poisson_manufactured}
u(\omega) = \overline{u}e^{\mu(\omega)}, \quad \overline{u}(x) := -\frac{1}{6}x (x^2-1).
\ee 
As a consequence, there holds that:
\be 
\label{eq:OperatorLaplace}
\fA_D^{(\tk)} \mM^k[u] = \mM^k[f_D] = c_\sigma\overline{f}_D^{(\tk)}\quad \text{with} \quad c_\sigma : =  \mM^k[e^{\mu(\omega)}] = e^{\frac{1}{2}\tk^2\sigma^2},
\ee
yielding
$$
\mM^k[u] = c_\sigma\overline{u}^{(\tk)}.
$$
Notice that $\tk=1$ simplifies to $\fA_D\overline{u} = \overline{f}_D$. We enforce the Dirichlet BCs throughout the transformation:
\be\label{eq:transfoPoisson}
\hat{\Sigma}^{\tk} = \left(\prod_{i=1}^\tk x_i (x_i^2-1)\right) \Sigma^{\tk},
\ee
and solve \eqref{eq:OperatorLaplace} for $\tk=1,\cdots,4$ and $\sigma = 1$ via TPINNs. In \Cref{fig:sol1} we represent $\text{diag}(\Sigma^\tk_\theta)$. In the same fashion, we represent the pointwise residual error in \Cref{fig:err1}.
\begin{figure}[htb!]
\center
\includegraphics[width=1\textwidth]{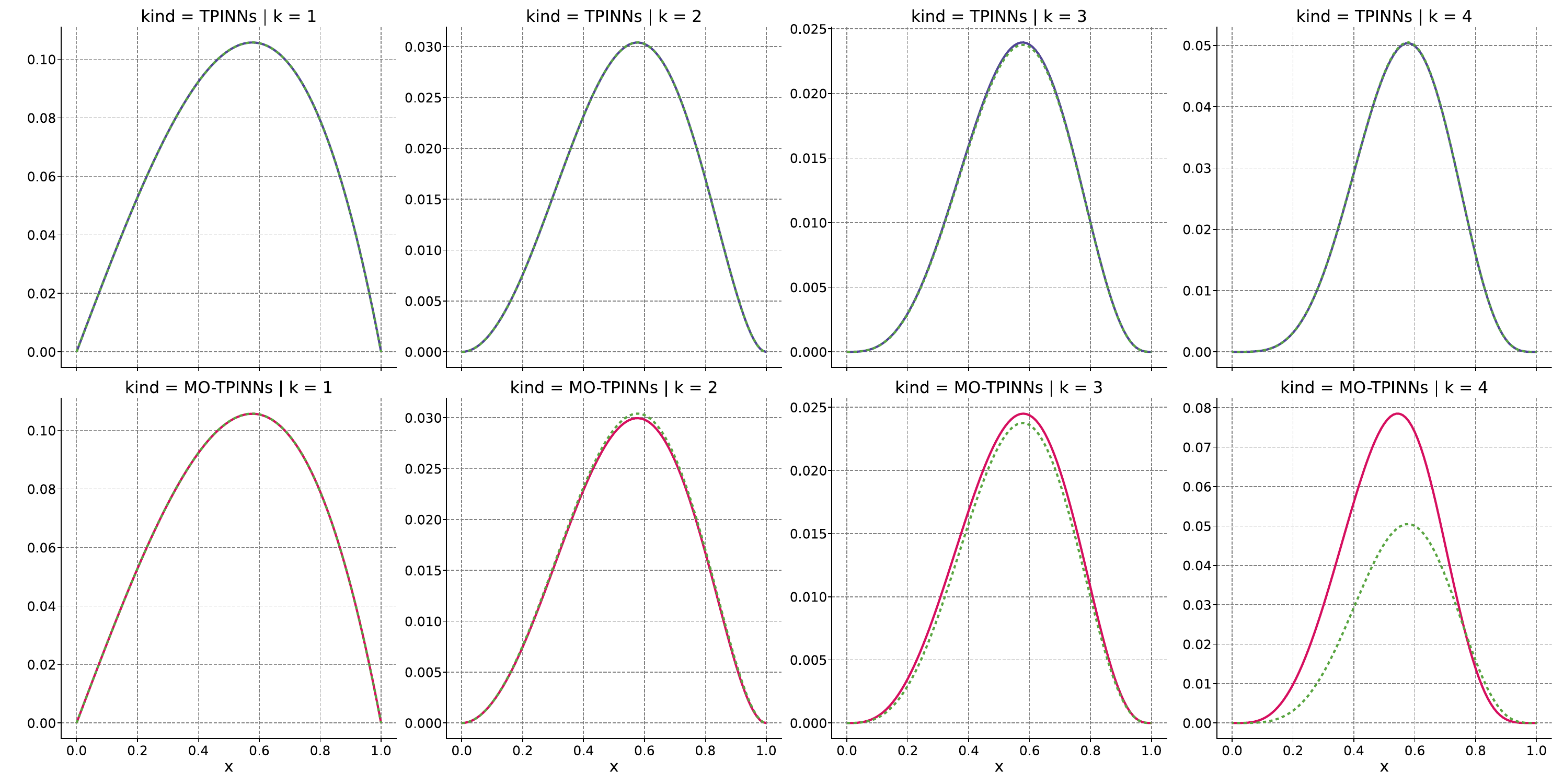}
\caption{Poisson 1D. $\text{diag}(\Sigma^\tk_\theta)$ for TPINNs (up, purple) and MO-TPINNs (bottom, pink) for $\tk=1,\cdots,4$. The exact solution is plotted with dashed green lines.}
\label{fig:sol1}       
\end{figure}

\begin{figure}[htb!]
\center
\includegraphics[width=1\textwidth]{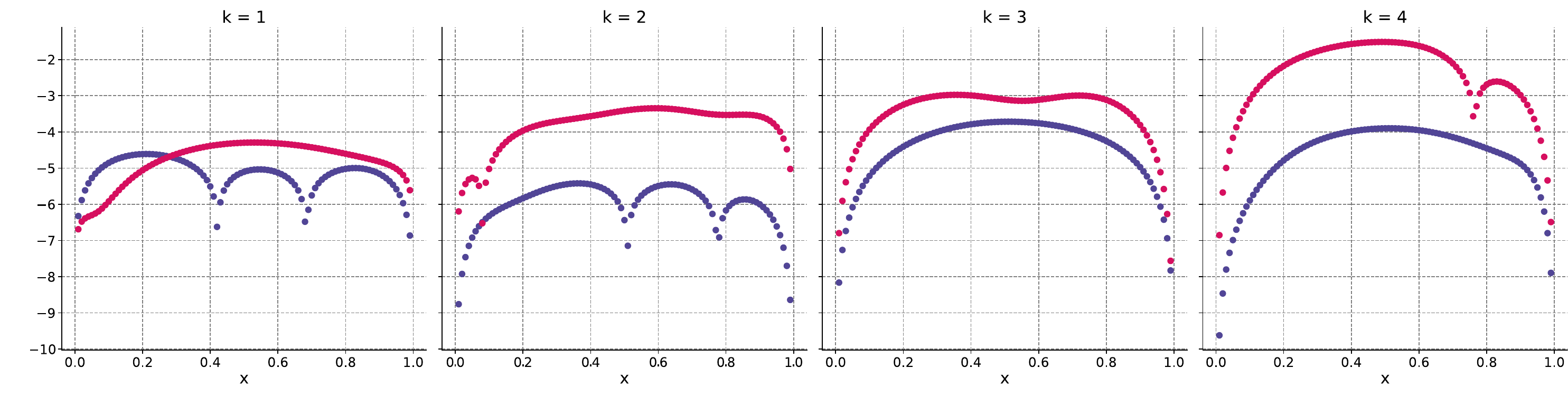}
\caption{Poisson 1D. Pointwise error (log-scale) for $\text{diag}(\Sigma^\tk_\theta)$ for TPINNs (purple) and MO-TPINNs (pink) for $\tk=1,\cdots,4$.}
\label{fig:err1}       
\end{figure}

In \Cref{tab:OverviewPoisson1D}, we \rev{summarize} the results for TPINNs. We remark that the training and test losses are increasing with $\tk$. Surprisingly, the error is low and remains stable with $\tk$. TPINNs provide a precise approximation with $N=4000$ training points and almost identical neural networks (\rev{see} $|\Theta|$). \rev{However, as} expected, the curse of dimensionality stems from higher\rev{-}order derivatives, inducing a strong increase in training times, e.g.,~$5246.1$s., i.e.~$1$ hour and $27$ minutes for $\tk=4$. According to these results, TPINNs prove \rev{to be} an efficient and robust technique \rev{for} approximat\rev{ing} the statistical moments. TPINNs deliver a surprisingly precise approximation to the tensor equation.

\begin{table}[ht!]
\renewcommand\arraystretch{1.5}
\begin{center}
\footnotesize
\begin{tabular}{
    |>{\centering\arraybackslash}m{.5cm}
    ||>{\centering\arraybackslash}m{2cm}
    |>{\centering\arraybackslash}m{2cm}
    |>{\centering\arraybackslash}m{2cm}
    ||>{\centering\arraybackslash}m{2cm}
    |>{\centering\arraybackslash}m{2cm}|
    }
    \hline
$\tk$ & $\mL_\theta $ & $\mL_\theta^\text{test}$ & Error & Time (s) & $|\Theta|$ \\ \hline\hline
1 & $2.41\times 10^{-8}$ &$2.30\times 10^{-8}$ & $2.00\times 10^{-5}$ & $30.8$ & $7801$\\\hline 
2 & $1.46\times 10^{-5}$ &$1.73\times 10^{-5}$ & $1.49\times 10^{-4}$ & $112.0$ & $7851$\\\hline 
3 & $1.03\times 10^{-3}$ &$1.40\times 10^{-3}$ & $6.07\times 10^{-4}$ & $644.4$ & $7901$\\\hline 
4 & $2.07$ &$3.37\times 10^{2}$ & $6.15\times 10^{-4}$ & $5246.1$ & $7951$\\\hline 
 \end{tabular}
\end{center}
\caption{Poisson 1D. Overview of the results for V-TPINNs.}
 \label{tab:OverviewPoisson1D} 
\end{table}
In the same fashion, we represent the results for MO-TPINNs in \Cref{tab:OverviewPoisson1DMO}. Here, the method shows controlled computational cost, at the expense of adding more terms to the function loss, and learning extra variables. We observe that training times remain stable with $\tk$. However, acknowledge that losses and errors grow with $\tk$. For $\tk=4$, MO-TPINNs yield a relative error of $23,4\%$. Notice in \eqref{fig:sol1} that MO-TPINNs fail at representing accurately the diagonal terms. 
\begin{table}[ht!]
\renewcommand\arraystretch{1.5}
\begin{center}
\footnotesize
\begin{tabular}{
    |>{\centering\arraybackslash}m{.5cm}
    ||>{\centering\arraybackslash}m{2cm}
    |>{\centering\arraybackslash}m{2cm}
    |>{\centering\arraybackslash}m{2cm}
    ||>{\centering\arraybackslash}m{2cm}
    |>{\centering\arraybackslash}m{2cm}|
    }
    \hline
$\tk$ & $\mL_\theta $ & $\mL_\theta^\text{test}$ & Error & Time (s) & $|\Theta|$ \\ \hline\hline
1 & $6.21\times 10^{-7}$ &$6.26\times 10^{-7}$ & $2.03\times 10^{-4}$ & $36.6$ & $7801$\\\hline 
2 & $1.36\times 10^{-2}$ &$1.35\times 10^{-2}$ & $1.10\times 10^{-2}$ & $48.9$ & $7902$\\\hline 
3 & $1.07\times 10^{-1}$ &$1.32\times 10^{-1}$ & $4.86\times 10^{-2}$ & $67.9$ & $8003$\\\hline 
4 & $2.00\times 10^{3}$ &$6.65\times 10^{4}$ & $2.34\times 10^{-1}$ & $88.5$ & $8104$\\\hline 
 \end{tabular}
\end{center}
\caption{Poisson 1D. Overview of the results for MO-TPINNs.}
 \label{tab:OverviewPoisson1DMO} 
\end{table}

\subsection{Stationary 1D Schrödinger}\label{subsec:StationarySchrodinger}In this case, we add a cubic non-linear term to \Cref{subsec:Poisson}. For any $\uplambda \neq0$ and $\kappa \in \IR$, we consider the one-dimensional (non-linear) time-harmonic Schrödinger equation \cite{perez2019stationary}:
\be\label{eq:SE}
-\partial^2_{xx} u(\omega) - \uplambda u^3(\omega) - \kappa^2 u(\omega)   = f_D(\omega) \quad\text{for}\quad x\in D=(0,1)\quad \text{with}\quad u(\omega)|_{\Gamma}=0 \quad \IP\text{-a.e.}\in \Omega.
\ee 
We set $\uplambda=1$ and $\kappa=0$, and reuse the manufactured solution $u(\omega)$ as in \eqref{eq:poisson_manufactured}. Thus, \rev{we have} the source term $f_D(\omega)= \overline{f}_D e^{\mu(\omega)}$ with
$$
\overline{f}_D(x) = x - \uplambda \overline{u}^3(x) \quad \text{in} \quad D.
$$
Notice that for $\tk=1$, \eqref{eq:opeq} yields $\fA_D \overline{u} = \overline{f}_D$. For $\tk=2$, we set $S: u \mapsto u^3$. There holds that for $x,y\in D$:
\begin{align*}
(\fA_D \otimes \fA_D ) \Sigma (x,y) & = (-\partial^2_{xx} - \uplambda S^3)\otimes (-\partial^2_{yy} - \uplambda S^3)\Sigma(x,y)\\
& = \partial^2_{xx}\partial^2_{yy} \Sigma(x,y) + \partial^2_{xx}\Sigma^3(x,y)+ (\partial^2_{yy}\Sigma)^3(x,y) + \Sigma^6(x,y).
\end{align*}
We set:
\be\label{eq:SchrExact}
\Sigma_D(x,y) = c_\sigma  v(x) w(y)
\ee 
with $v=w=\overline{u}$, and $c_\sigma$ in \eqref{eq:OperatorLaplace}. We aim at obtaining the manufactured right-hand side corresponding to \eqref{eq:SchrExact}. First, acknowledge that:
\begin{align*}
\partial^2_{xx}v^3(x) & = \partial_x ( 3 \partial_x v(x) v(x) )\\
&=  3  \partial_{xx}^2 v(x) v(x) + 3  (\partial_x v(x))^2\\
& = 3  \left[- x v(x) +  \left(-\frac{x^2}{2} +\frac{1}{6}\right)^2 \right].
\end{align*}
Second, 
$$
\partial^2_{xx}\partial^2_{yy} \Sigma (x,y)  =c_\sigma  xy ,\quad (\partial^2_{yy}\Sigma(x,y))^3  = -c_\sigma^3v(x)^3 y^3\quad \text{and}\quad \Sigma^6(x,y)= c_\sigma^6v(x)^3 w(y)^3.
$$
As a consequence:
\begin{align*}
(\fA_D\otimes \fA_D) \Sigma(x,y) & = c_\sigma x y +   3c_\sigma^3 \left[ - x v(x)+  \left(-\frac{x^2}{2} +\frac{1}{6}\right)^2\right]w(y)^3 - c_\sigma^3v(x)^3 y^3 + c_\sigma^6v(x)^3 w(y)^3 \\
& =:\tC_D(x,y) \equiv \mM^2[f_D].
\end{align*}
We use the transformation in \eqref{eq:transfoPoisson} and apply TPINNs to \eqref{eq:SchrExact} with the parameters in \Cref{tab:OverviewHP}. As regards MO-PINNs we set:
\be 
\begin{cases}
-\partial^2_{xx} \tV_0 - \uplambda\tV_0^3 & = \tC_D \\
-\partial^2_{yy} \tV_1 - \uplambda \tV_1^3 & = \tV_0.
\end{cases}
\ee
We represent $\text{diag}(\Sigma_\theta)$ and the pointwise error in \Cref{fig:sol2} for both TPINN architecture\rev{s}. On the left-hand figure, we portray the training collocation points. Both models achieved to manage and learn the non-linearity. To our knowledge, this is the first time that non-linear tensor operator equation is solved numerically. TPINNs give an accurate solution, with pointwise error being lesser than $2 \times 10^{-5}$ (resp.~$4 \times 10^{-4}$) for V-TPINN (resp.~MO-TPINN).
\begin{figure}[htb!]
\center
\includegraphics[width=1\textwidth]{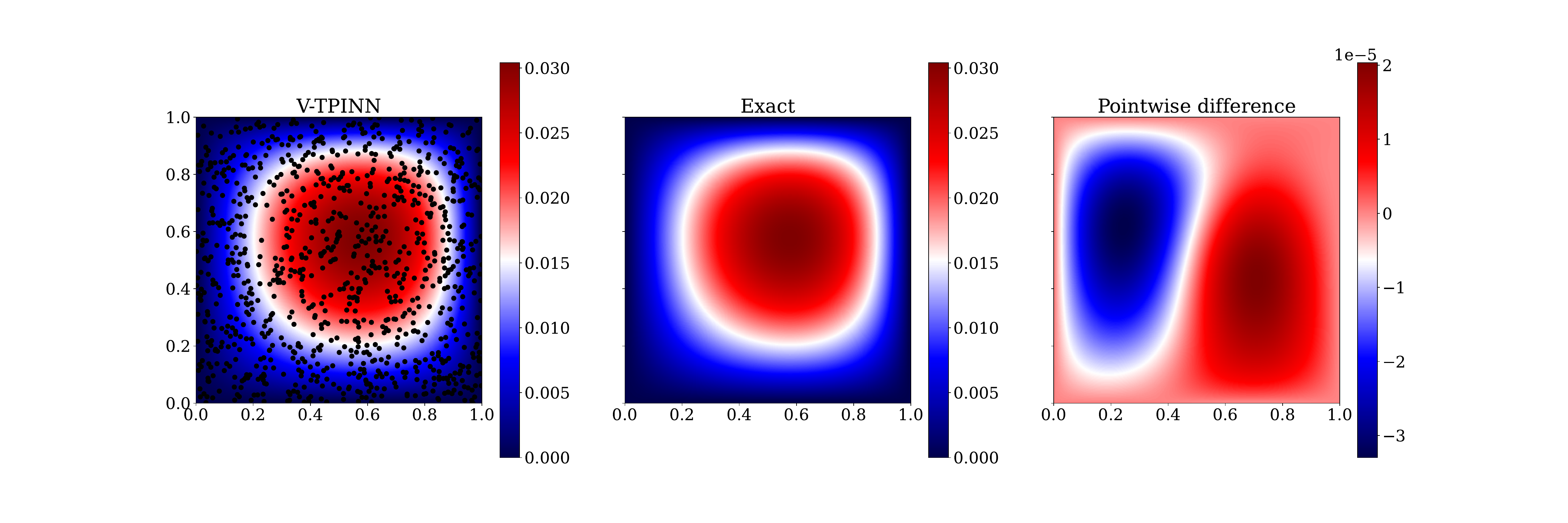}
\vspace{-.5cm}
\includegraphics[width=1\textwidth]{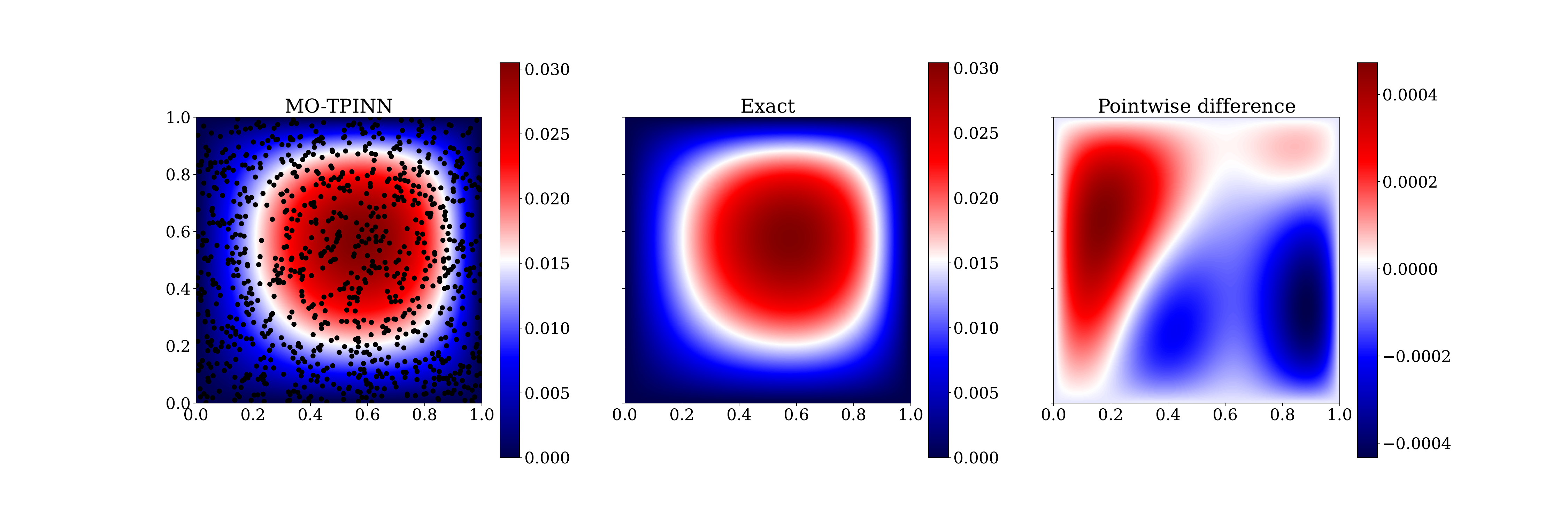}
\caption{Schrödinger 1D. V-TPINN (up) and MO-TPINN (bottom) for $\Sigma_\theta^\tk$. We represent the solution along with the training collocation points (left). The exact solution is represented (center) and the pointwise difference between the TPINN and exact solution (right).}
\label{fig:sol2}       
\end{figure}
To go further, the results are summarized in \Cref{tab:OverviewSchrodinger1D}. We notice that all models have low training errors and generalize well ($\mL_\theta \approx \mL_\theta^\text{test}$). V-TPINN gives a more precise solution than MO-TPINN despite showing similar test loss. We remark that both TPINNs provide a highly accurate approximation in less than $150$ seconds, which is remarkable. 
\begin{table}[ht!]
\renewcommand\arraystretch{1.5}
\begin{center}
\footnotesize
\begin{tabular}{
    |>{\centering\arraybackslash}m{.5cm}
    ||>{\centering\arraybackslash}m{2cm}
    ||>{\centering\arraybackslash}m{2cm}
    |>{\centering\arraybackslash}m{2cm}
    |>{\centering\arraybackslash}m{2cm}
    ||>{\centering\arraybackslash}m{2cm}
    |>{\centering\arraybackslash}m{1.5cm}|
    }
    \hline
$\tk$ & Architecture & $\mL_\theta $ & $\mL_\theta^\text{test}$ & Error & Time (s) & $|\Theta|$ \\ \hline\hline
1 & - & $9.41\times 10^{-7}$ &$9.42\times 10^{-7}$ & $2.42\times 10^{-4}$ & $35.9$ & $7801$\\\hline \hline 
2 & V-TPINN & $4.49\times 10^{-5}$ &$4.97\times 10^{-5}$ & $9.71\times 10^{-4}$ & $146.8$ & $7851$\\\hline
2 & MO-TPINN & $4.73\times 10^{-5}$ &$5.54\times 10^{-5}$ & $1.23\times 10^{-2}$ & $65.4$ & $7902$\\\hline
 \end{tabular}
\end{center}
\caption{Schrödinger 1D. Overview of the results.}
 \label{tab:OverviewSchrodinger1D} 
\end{table}

\subsection{Helmholtz equation in two dimensions}\label{subsec:Helmholtz}We aim at verifying that TPINNs can 
 be applied to more complex problems, particularly,~for $d > 1$. We consider the solution of the Helmholtz equation over $D:= [0,1]^d$, $d=2$ with $n\in \IN_1$ and $\kappa := 2 \pi n$ such that
$$
-\Delta u(\omega) - \kappa^2 u (\omega)= f_D(\omega)\quad\text{with}\quad u(\omega)|_\Gamma=0\quad \IP\text{-a.e. }\omega \in \Omega.
$$
For $\tk=1$, we set
$$\label{eq:loadHelmholtz}
\IE[f_D] = \overline{f}_D(x,y) := \kappa^2 \sin(n \pi x) \sin (n \pi y),
$$
and use the transformation $(x,y) \mapsto x (x-1) y (y-1)$ to enforce the BCs. 

Let us focus now on $\tk=2$ and $\Sigma = \mM^2[u]$. One has that:
\begin{align*}
(\fA_D \otimes \fA_D) \Sigma &= (-\Delta_{\bx_1} - \kappa^2 ) \otimes(-\Delta_{\bx_2} - \kappa^2)\Sigma \\
& =  \Delta_{\bx_1}\Delta_{\bx_2} \Sigma + \kappa^2 \Delta_{\bx_1} \Sigma  + \kappa^2 \Delta_{\bx_2} \Sigma + \kappa^4 \Sigma .
\end{align*}
Boundary conditions are enforced via the transformation: 
\be
\hat{\Sigma} : = x (x-1) y (y-1) z (z-1) t (t-1) \Sigma.
\ee
Lastly, MO-TPINNs in \eqref{eq:MOPINNssol} rewrite as:
\be 
\begin{cases}
-\Delta_{\bx_1}\tV_0 - \kappa^2 \tV_0 &= \tC_D,\\
-\Delta_{\bx_2}\tV_1- \kappa^2 \tV_1 &=\tV_0.
\end{cases}
\ee
We apply TPINNs with the parameters in \Cref{tab:OverviewHP} for separable right-hand side, and exponential covariance kernel.
\subsubsection{Separable right-hand side}\label{subsubsec:HelmholtzSep}
We set $\bx_1 = (x,y)$, $\bx_2 = (z,t)$ and define the following separable right-hand side:
\be \label{eq:CovSeparable}
\tC_D (x,y,z,t) := \kappa^4 \sin(n \pi x)\sin(n\pi y)\sin(n \pi z)\sin(n\pi t)=\overline{f}_D\otimes \overline{f}_D
\ee
with $\overline{f}_D$ in \eqref{eq:loadHelmholtz}. We represent $\text{diag}(\Sigma^\tk_\theta)$ for V-TPINNs and MO-TPINNs in \Cref{fig:solhelmV} and \Cref{fig:solhelmTP}, respectively. As in previous subsections, TPINNs yield an accurate approximation to $\text{diag}(\Sigma^\tk_\theta)$, with pointwise error lesser than $0.04$ (resp.~$0.06$) for V-TPINNs (resp.~MO-TPINNs). It is worth to mention that $\|\Sigma\|_{L^\infty(D^{(\tk)})} = 1$.
\begin{figure}[htb!]
\center
\includegraphics[width=\textwidth]{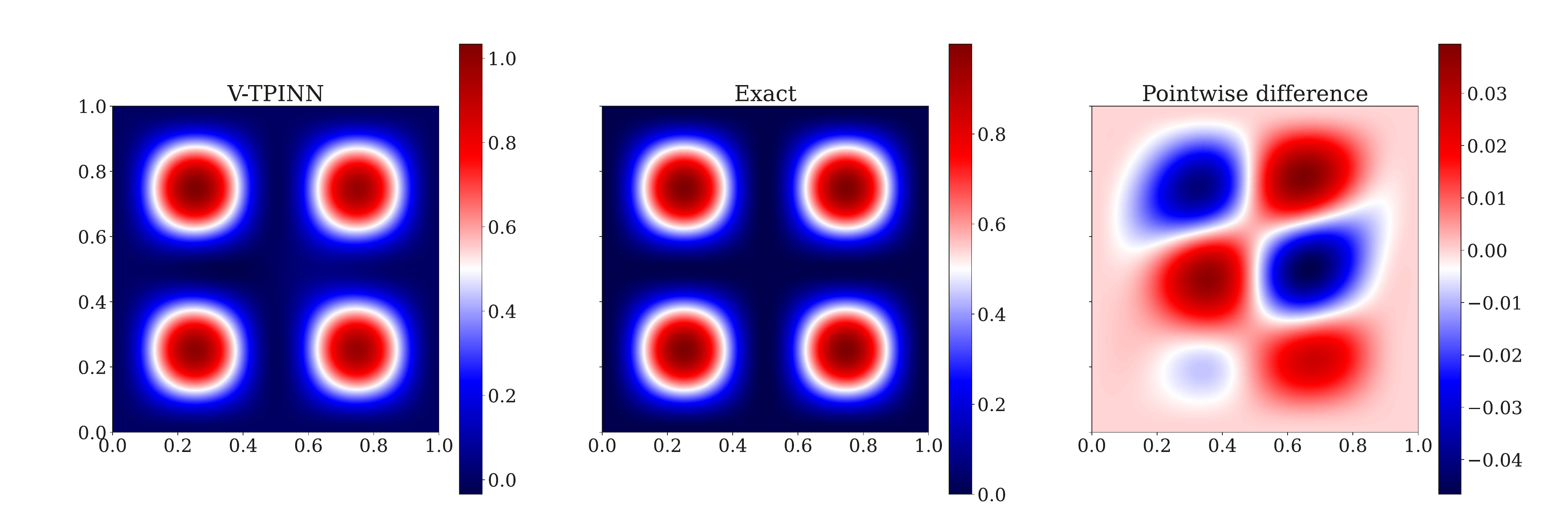}
\caption{Helmholtz 2D: Solution and pointwise error of V-TPINN for $\text{diag}(\Sigma^\tk_\theta)$ and a separable right-hand side.}
\label{fig:solhelmV}       
\end{figure}

\begin{figure}[htb!]
\center
\includegraphics[width=\textwidth]{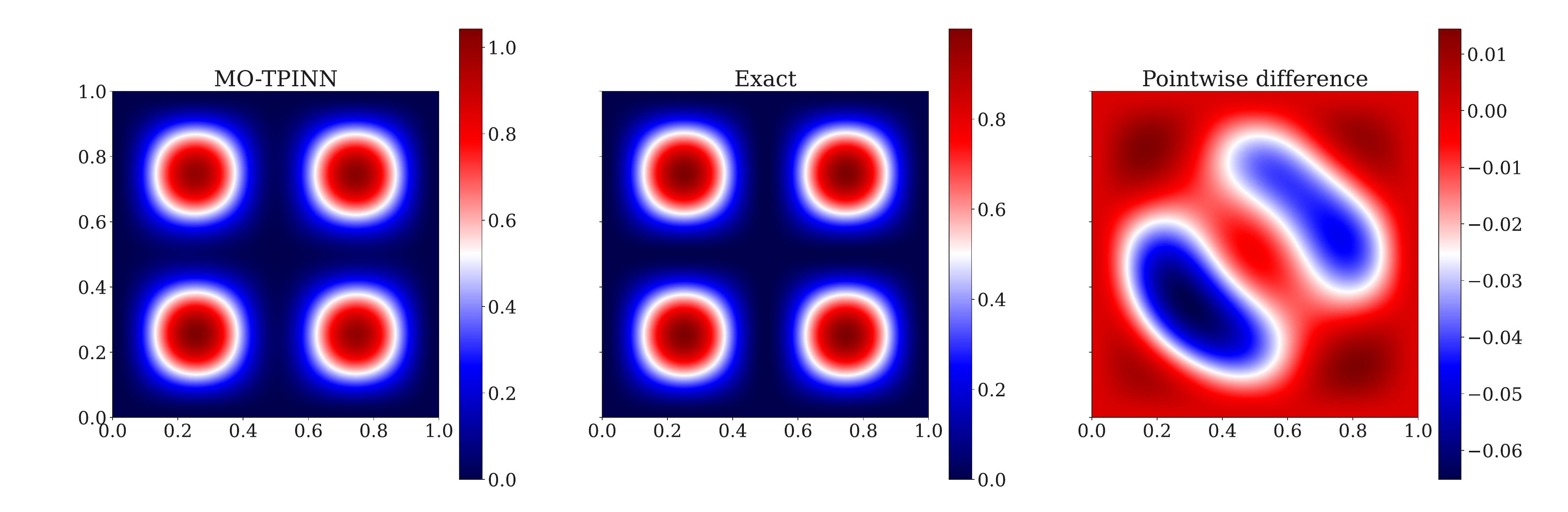}
\caption{Helmholtz 2D: Solution and pointwise error of MO-TPINN for $\text{diag}(\Sigma^\tk_\theta)$ and a separable right-hand side.}
\label{fig:solhelmTP}       
\end{figure}
In \Cref{tab:OverviewHelmholtz1}, we give further details on performance. We remark that: (i) the models train and generalize well; (ii) losses and error for both TPINNs are similar, and around $100$ times less precise than for $\tk=1$. Furthermore, MO-TPINNs yield similar training time to case $\tk=1$, for an error of $7.09\%$, outperforming V-TPINNs in this case: the error for MO-TPINNs is $22\%$ higher than for V-TPINNs, for a $5$-fold faster training. 

\begin{table}[ht!]
\renewcommand\arraystretch{1.5}
\begin{center}
\footnotesize
\begin{tabular}{
    |>{\centering\arraybackslash}m{.5cm}
    ||>{\centering\arraybackslash}m{2cm}
    ||>{\centering\arraybackslash}m{2cm}
    |>{\centering\arraybackslash}m{2cm}
    |>{\centering\arraybackslash}m{2cm}
    ||>{\centering\arraybackslash}m{2cm}
    |>{\centering\arraybackslash}m{1.5cm}|
    }
    \hline
$\tk$ & Architecture & $\mL_\theta $ & $\mL_\theta^\text{test}$ & Error & Time (s) & $|\Theta|$ \\ \hline\hline
1 & - & $5.93\times 10^{-4}$ &$5.41\times 10^{-4}$ & $4.52\times 10^{-4}$ & $159.3$ & $124251$\\\hline \hline 
2 & V-TPINN & $1.33\times 10^{1}$ &$1.74\times 10^{1}$ & $5.82\times 10^{-2}$ & $814.2$ & $124951$\\\hline
2 & MO-TPINN & $3.59\times 10^{1}$ &$4.05\times 10^{1}$ & $7.09\times 10^{-2}$ & $147.3$ & $125302$\\\hline 
 \end{tabular}
\end{center}
\caption{Helmholtz 2D: Overview of the results for a simple right-hand side.}
 \label{tab:OverviewHelmholtz1} 
\end{table}

\subsubsection{Exponential covariance kernel}\label{subsubsec:HelmholtzExp}Separable right-hand sides allow to use a manufactured solution. Once convergence is verified, application of TPINNs to general right-hand side is straightforward. We introduce the squared exponential (Gaussian) and $1$-exponential covariance functions $f_D \sim \mG\mP (\overline{f}_D , \tC_D)$ \cite[Fig.~2]{MIKAMatrixFreeKL} with $\overline{f}_D$ in \eqref{eq:loadHelmholtz} and 
\be\label{eq:kernelExponential}
\tC_D(\bx_1,\bx_2) = \sigma^2 \exp\left(- \frac{|x-z|^2}{2 \lambda}\right)\quad \text{and}\quad \tC (\bx_1,\bx_2) = \sigma^2 \exp \left(- \frac{|x - z|}{2 \lambda}\right)
\ee 
for $\bx_1 =(x,y)$, $\bx_2=(z,t)$, $\sigma\in \IR$ and $\lambda  >0$. We apply TPINNs again for $\sigma = \kappa^2$ and $\lambda = 5$. To begin with, in \Cref{fig:kernelshelm} we represent $\text{diag}(\tC_D)$ for the three right-hand sides considered. 

\begin{figure}[htb!]
\center
\includegraphics[width=\textwidth]{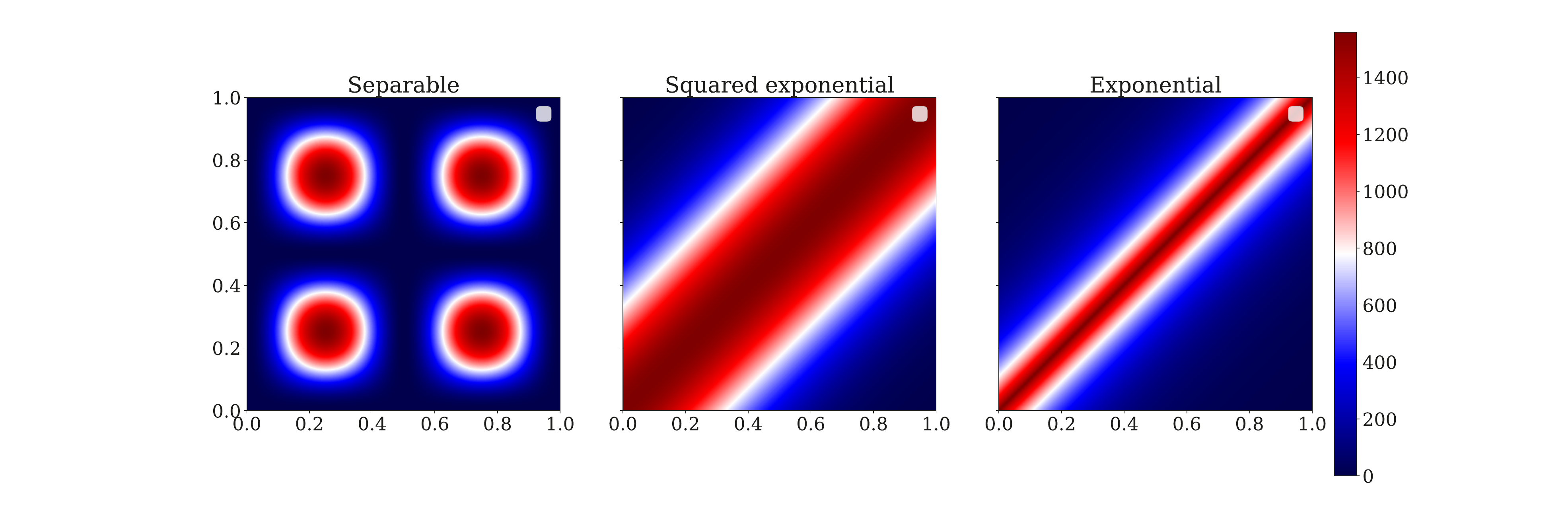}
\caption{Helmholtz 2D: Plot of $\text{diag}(\tC_D)$ defined in \eqref{eq:CovSeparable} and \eqref{eq:kernelExponential}.}
\label{fig:kernelshelm}       
\end{figure}

Next, we plot the diagonal of V-TPINN and MO-TPINN for the Gaussian kernel in \Cref{fig:Gaussiansol}. Remark that the solutions differ. Likewise, the diagonal for exponential kernel is represented in \Cref{fig:Exponentialsol}. This time, the solutions have a more similar pattern. To give a further insight on these simulations, we summarize the results in \Cref{tab:OverviewHelmholtzKernel}. First, we remark that V-TPINNs failed to train in the Gaussian case ($\mL_\theta = 1.34 \times 10^{3}$ while $\mL_\theta^\text{test} = 1.08 \times 10^6$). This behavior was quite unexpected, as Gaussian kernels are smoother than $1$-exponential kernels. Training losses are $\sim 10^3$ for exponential kernel, as compared to $\sim 10^1$ for separable right-hand sides in \Cref{tab:OverviewHelmholtz1}: training loss is sensitive to the right-hand side.

\begin{figure}[htb!]
\begin{subfigure}{.5\textwidth}
\center
\includegraphics[width=\textwidth]{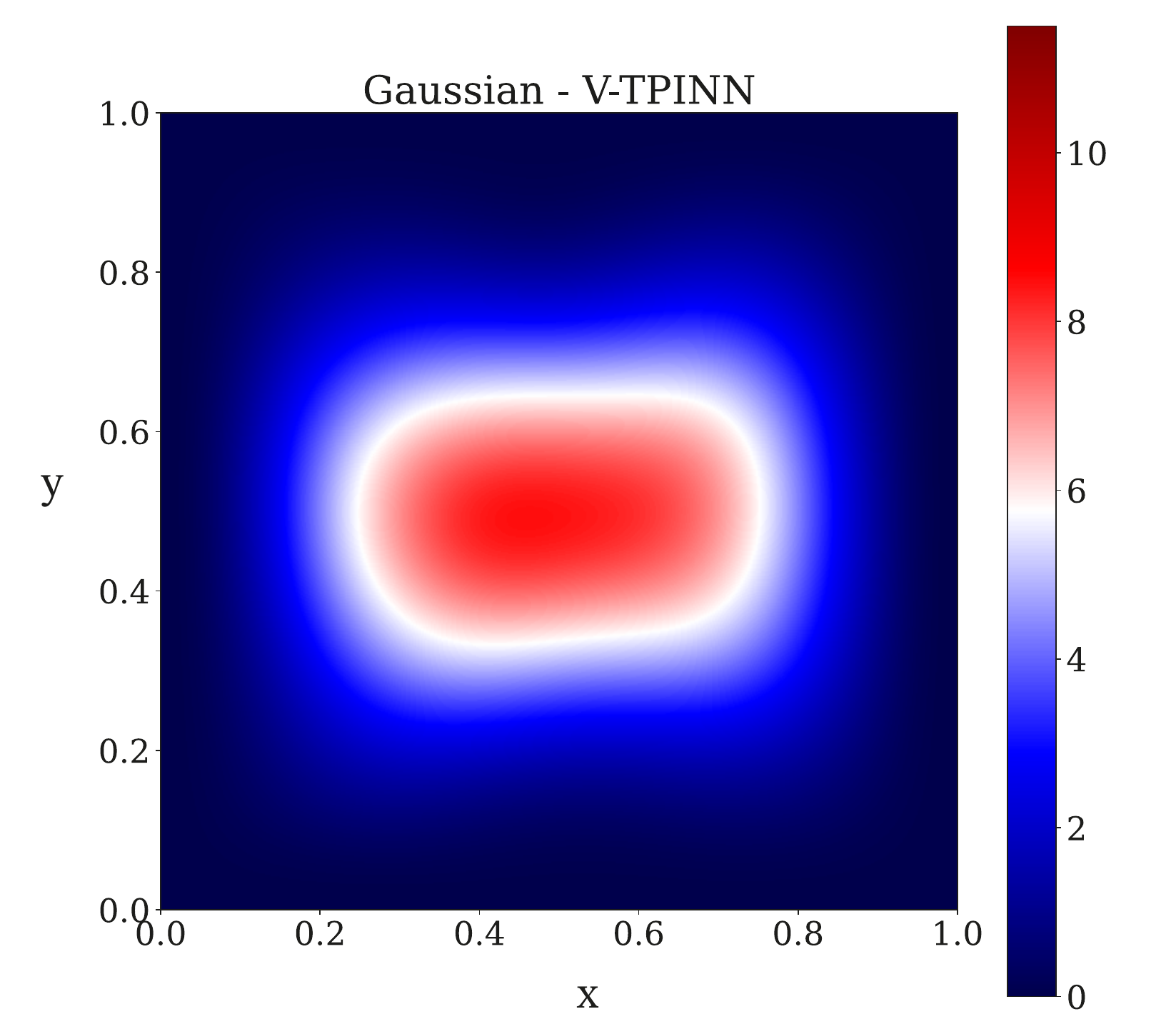}
\end{subfigure}
\begin{subfigure}{.5\textwidth}
\includegraphics[width=\textwidth]{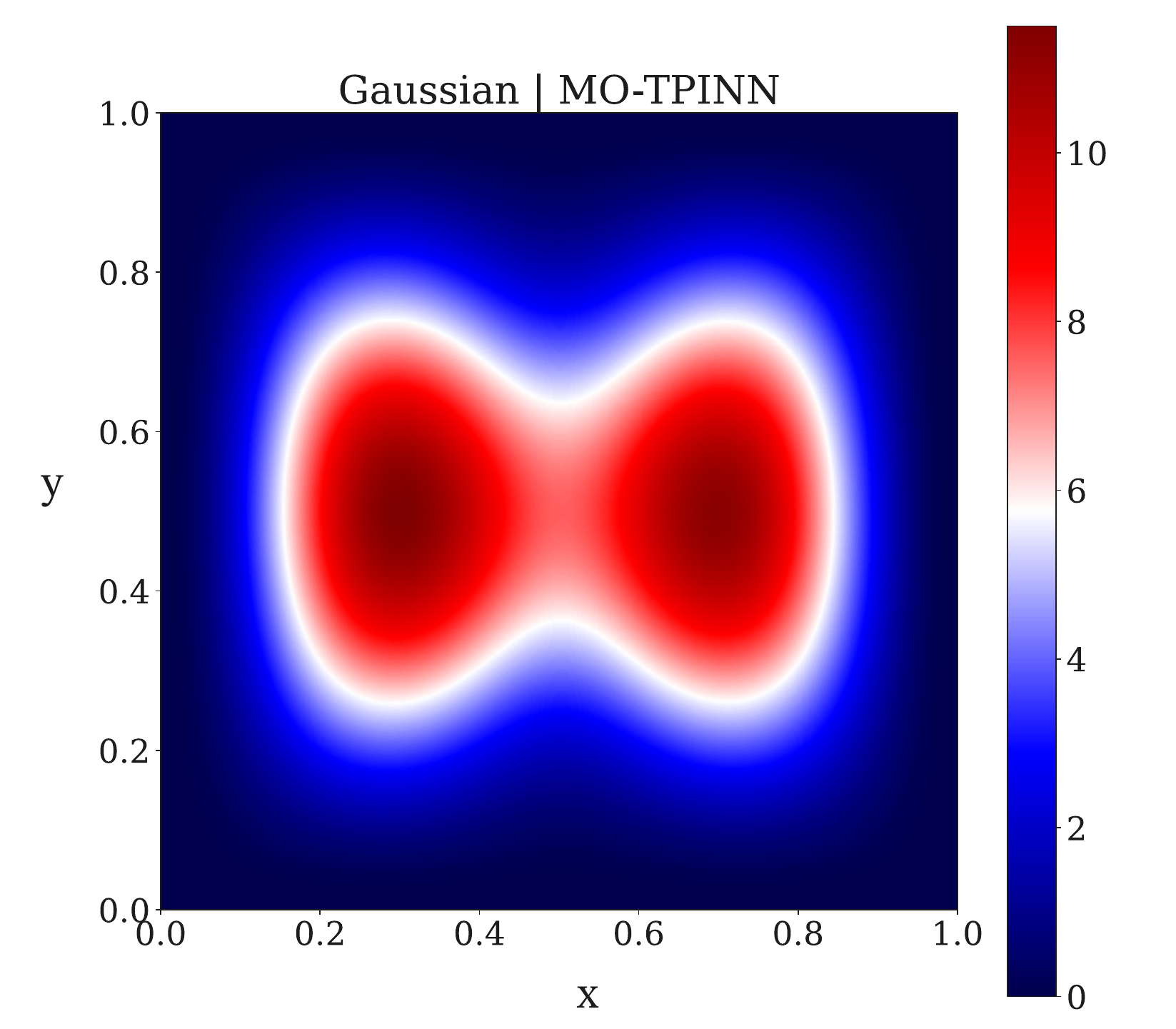}
\end{subfigure}
\caption{Helmholtz 2D: Solution for the Gaussian kernel in \eqref{eq:kernelExponential}.}
\label{fig:Gaussiansol}       
\end{figure}

\begin{figure}[htb!]
\begin{subfigure}{.5\textwidth}
\center
\includegraphics[width=\textwidth]{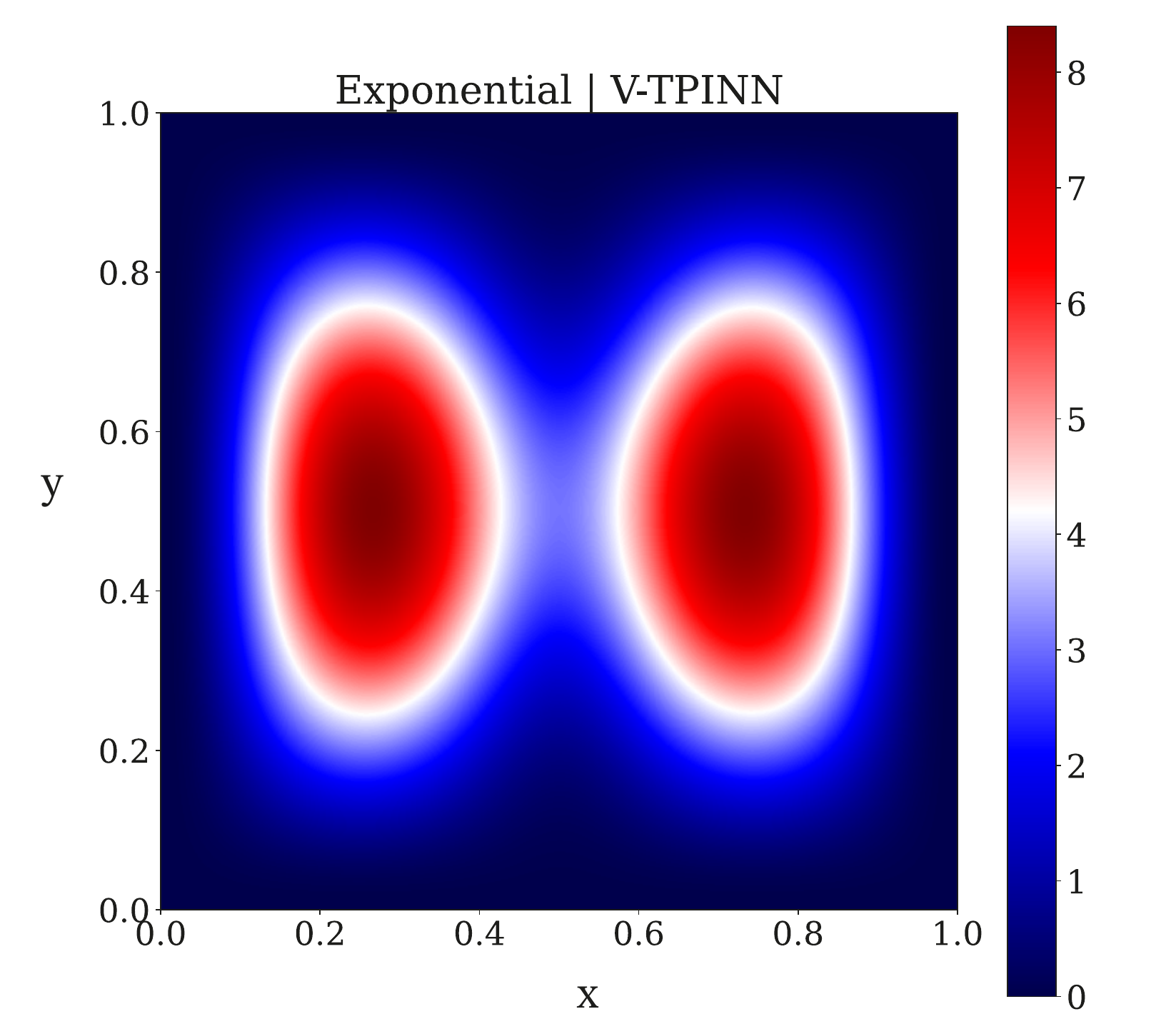}
\end{subfigure}
\begin{subfigure}{.5\textwidth}
\includegraphics[width=\textwidth]{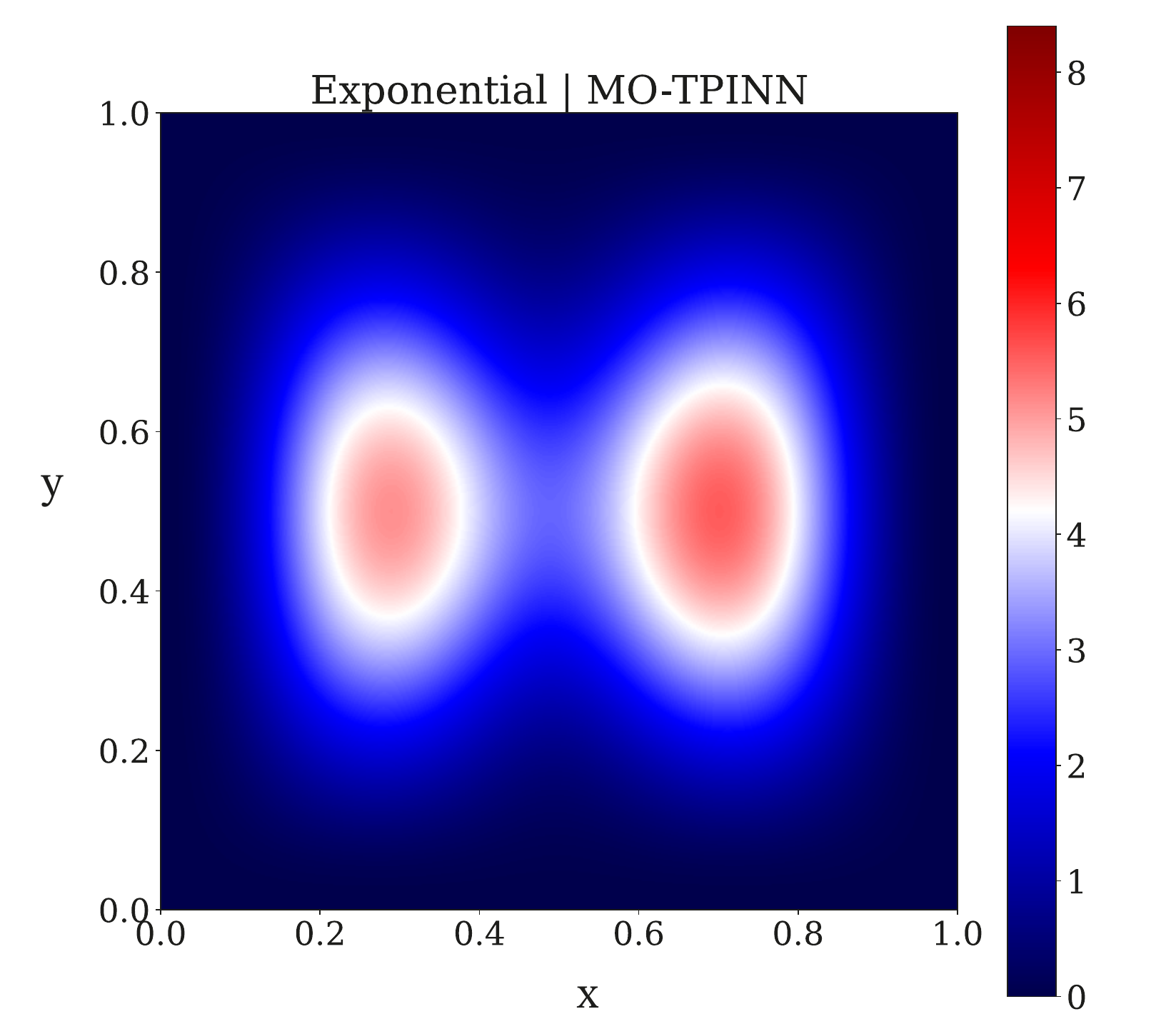}
\end{subfigure}
\caption{Helmholtz 2D: Solution for the $1$-exponential kernel in \eqref{eq:kernelExponential}.}
\label{fig:Exponentialsol}       
\end{figure}

\begin{table}[ht!]
\renewcommand\arraystretch{1.5}
\begin{center}
\footnotesize
\begin{tabular}{
    |>{\centering\arraybackslash}m{.5cm}
    ||>{\centering\arraybackslash}m{2cm}
    ||>{\centering\arraybackslash}m{2cm}
    |>{\centering\arraybackslash}m{2cm}
    |>{\centering\arraybackslash}m{2cm}
    ||>{\centering\arraybackslash}m{2cm}
    |>{\centering\arraybackslash}m{1.5cm}|
    }
    \hline
\multirow{2}{*}{$\tk$} &  \multirow{2}{*}{Architecture} & \multicolumn{2}{|c|}{Gaussian} & \multicolumn{2}{|c|}{Exponential}\\ \cline{3-6}
&&   $\mL_\theta $ & $\mL_\theta^\text{test}$ &$\mL_\theta $ & $\mL_\theta^\text{test}$ \\ \hline \hline
2 & V-TPINN & $1.34\times 10^{+3}$ &$1.08\times 10^{+6}$ & $2.13\times 10^{+3}$ & $2.02\times 10^{+4}$ \\\hline
2 & MO-TPINN & $5.66\times 10^{+3}$ &$6.77\times 10^{+3}$ & $1.32\times 10^{+3}$ & $1.59\times 10^{+3}$ \\\hline
 \end{tabular}
\end{center}
\caption{Helmholtz 2D: Overview of the results for exponential covariance kernels.}
 \label{tab:OverviewHelmholtzKernel} 
\end{table}

\subsection{Heat equation}\label{subsec:Heat}
To finish, we apply the scheme to a time dependent operator. We consider the heat equation in $[0,L]\times[0,T]$:
\begin{align*}
\partial_t u(\omega) - a \Delta u(\omega) = f_D(\omega) \quad \text{with}\quad  u(\omega)|_\Gamma = 0\quad \text{and}\quad u(\omega, t= 0) = \sin(n \pi x) \quad \IP\text{-a.e. }\omega \in \Omega.
\end{align*}
wherein we set $L=T=n=1$ and $a=0.4$. We consider the exact solution for $\tk=1$ with $\IE[f]=0$:
\be 
\IE[u] = \overline{u}(x,t) = \sin(n \pi x / L) e^{-n \pi^2 a t} 
\ee
and apply the following transformation:
\be 
\hat{u}(x,t) = \sin(n \pi x) e^{- n \pi^2 a t}  + t u(x,t).
\ee

\subsubsection{Separable right-hand side}\label{subsubsec:HeatSeparable}
For $\tk=2$, we use the following manufactured solution: 
$$
\Sigma(x,t,y,z)= \overline{u}(x,t) \overline{u}(y,z).
$$
i.e
$$
\Sigma (x,t,y,t) = e^{- \pi^2 at} e^{- \pi^2a z} \sin(\pi x) \sin (\pi z)
$$
for $a=0.4$ and notice that:
\begin{align*}
(\fA_D \otimes \fA_D) \Sigma & = (\partial_t - a \Delta_{x})(\partial_z - a \Delta_y)\Sigma\\
&= \partial_t \partial_z \Sigma - a \partial_t \Delta_y \Sigma- a \Delta_x \partial_z \Sigma+ a^2 \Delta_x\Delta_y \Sigma\\
& = \partial^2_{tz} \Sigma - a \partial^3_{tyy} \Sigma - a \partial^3_{xxz} \Sigma + a^2 \partial^4_{xxyy} \Sigma.
\end{align*}

We set the transformation:
\be 
\hat{\Sigma} (x,t,y,z) = e^{- \pi^2t} e^{- \pi^2 z} \left(  \sin(\pi x) \sin (\pi z) +  t z \Sigma (x,t,y,z)\right) .
\ee 
For the MO-TPINNs, one has that:
\be 
\begin{cases}
-\partial_t \tV_0 - \partial^2_{xx}\tV_0  &= \tC_D,\\
-\partial_z \tV_1 -\partial^2_{yy}\tV_1 &=\tV_0.
\end{cases}
\ee 
with $\tC_D=0$ in this case (see definition in \Cref{subsubsec:HelmholtzExp} for the exponential case.) We solve the problem with parameters in \Cref{tab:OverviewHP}, and sum up results in \Cref{tab:OverviewHeat}. We remark that the neural nets train well and supply accurate solutions in a few minutes. It is worthy to notice that V-TPINN (resp.~MO-TPINN) yields a solution with relative error of $8.87 \times 10^{-4}$ (resp.~$6.60 \times 10^{-4}$) despite being trained over nets with less than $1000$ trainable parameters. This demonstrates the learning capacity of PINNs and TPINNs. 
\begin{table}[ht!]
\renewcommand\arraystretch{1.5}
\begin{center}
\footnotesize
\begin{tabular}{
    |>{\centering\arraybackslash}m{.5cm}
    ||>{\centering\arraybackslash}m{2cm}
    ||>{\centering\arraybackslash}m{2cm}
    |>{\centering\arraybackslash}m{2cm}
    |>{\centering\arraybackslash}m{2cm}
    ||>{\centering\arraybackslash}m{2cm}
    |>{\centering\arraybackslash}m{1.5cm}|
    }
    \hline
$\tk$ & Architecture & $\mL_\theta $ & $\mL_\theta^\text{test}$ & Error & Time (s) & $|\Theta|$ \\ \hline\hline
1 & - & $2.15\times 10^{-8}$ &$2.05\times 10^{-8}$ & $1.84\times 10^{-5}$ & $49.8$ & $921$\\\hline \hline 
2 & V-TPINN & $1.16\times 10^{-5}$ &$1.27\times 10^{-5}$ & $8.87\times 10^{-4}$ & $226.9$ & $961$\\\hline
2 & MO-TPINN & $1.02\times 10^{-6}$ &$1.03\times 10^{-6}$ & $6.40\times 10^{-3}$ & $74.6$ & $982$\\\hline 
 \end{tabular}
\end{center}
\caption{Heat 1D: Overview of the results for a simple right-hand side.}
 \label{tab:OverviewHeat} 
\end{table}
\subsubsection{Exponential covariance kernel}\label{subsubsec:HeatExp}
Next, we set the Gaussian and $1$-exponential covariance kernels for $\bx_1 = (x,t)$ and $\bx_2 = (y,z)$ as:
$$
\tC_D (\bx_1,\bx_2) = \sigma^2 \exp\left(- \frac{|x- y|^2}{2 \lambda }-2 (t-z)\right)  \quad \text{and}\quad \tC_D (\bx_1,\bx_2) = \sigma^2 \exp \left(- \frac{|x-y|}{2 \lambda}-2 (t-z)\right)
$$ 
respectively, with $\sigma = 10$ and $\lambda =5$. Acknowledge that TPINNs do not generalize as well as in previous sections, despite delivering low test losses of $7.71 \times 10^{-2}$ (resp.~$4.80 \times 10^{-4}$) for V-TPINN (resp.~MO-TPINN). It is important to notice that in this case, MO-TPINNs leads to an approximation with \rev{a} lower test error and training time (by a factor of $3$). To finish, we represent $\text{diag}(\Sigma_\theta)$ for both covariance kernels in \Cref{fig:Heat-Gaussiansol} and \Cref{fig:Heat-Exponentialsol}. Notice that the difference between solutions is low but noticeable in \Cref{fig:Heat-Gaussiansol}, while being imperceptible in \Cref{fig:Heat-Exponentialsol}. It is again interesting to notice that Gaussian kernel is harder to train that its $1$-exponential counterpart. 


\begin{table}[ht!]
\renewcommand\arraystretch{1.5}
\begin{center}
\footnotesize
\begin{tabular}{
    |>{\centering\arraybackslash}m{.5cm}
    ||>{\centering\arraybackslash}m{2cm}
    ||>{\centering\arraybackslash}m{2cm}
    |>{\centering\arraybackslash}m{2cm}
    |>{\centering\arraybackslash}m{2cm}
    ||>{\centering\arraybackslash}m{2cm}
    |>{\centering\arraybackslash}m{1.5cm}|
    }
    \hline
\multirow{2}{*}{$\tk$} &  \multirow{2}{*}{Architecture} & \multicolumn{2}{|c|}{Gaussian} & \multicolumn{2}{|c|}{Exponential}\\ \cline{3-6}
&&   $\mL_\theta $ & $\mL_\theta^\text{test}$ &$\mL_\theta $ & $\mL_\theta^\text{test}$ \\ \hline \hline
2 & V-TPINN & $1.17\times 10^{-2}$ &$7.71\times 10^{-2}$ & $1.26\times 10^{-2}$ &$5.50\times 10^{-2}$ \\\hline
2 & MO-TPINN & $4.80\times 10^{-4}$ &$8.04\times 10^{-4}$ & $6.64\times 10^{-4}$ & $1.01\times 10^{-3}$ \\\hline
 \end{tabular}
\end{center}
\caption{Heat 1D: Overview of the results for exponential covariance kernels.}
 \label{tab:OverviewHeatKernel} 
\end{table}

\begin{figure}[htb!]
\begin{subfigure}{.5\textwidth}
\center
\includegraphics[width=\textwidth]{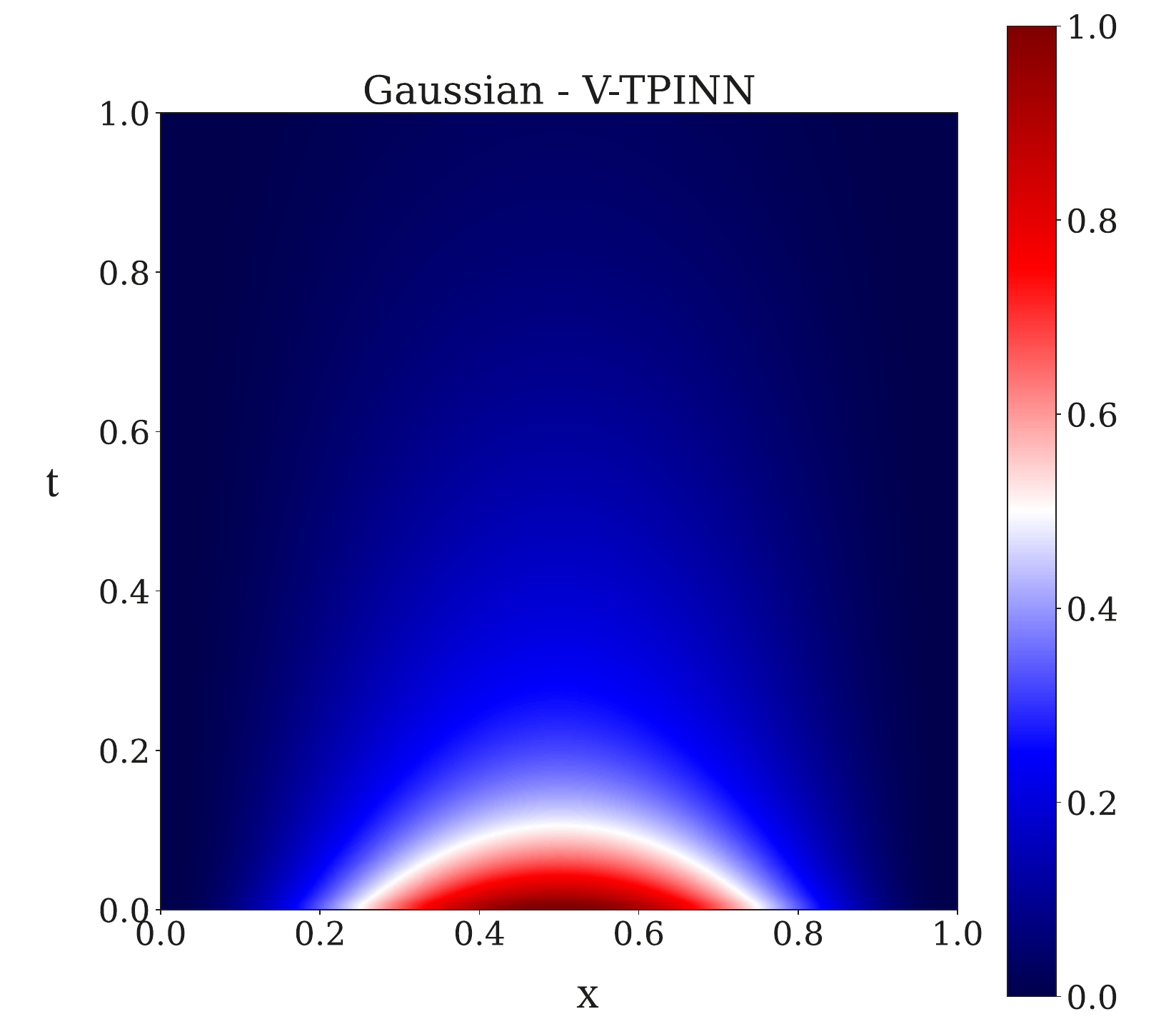}
\end{subfigure}
\begin{subfigure}{.5\textwidth}
\includegraphics[width=\textwidth]{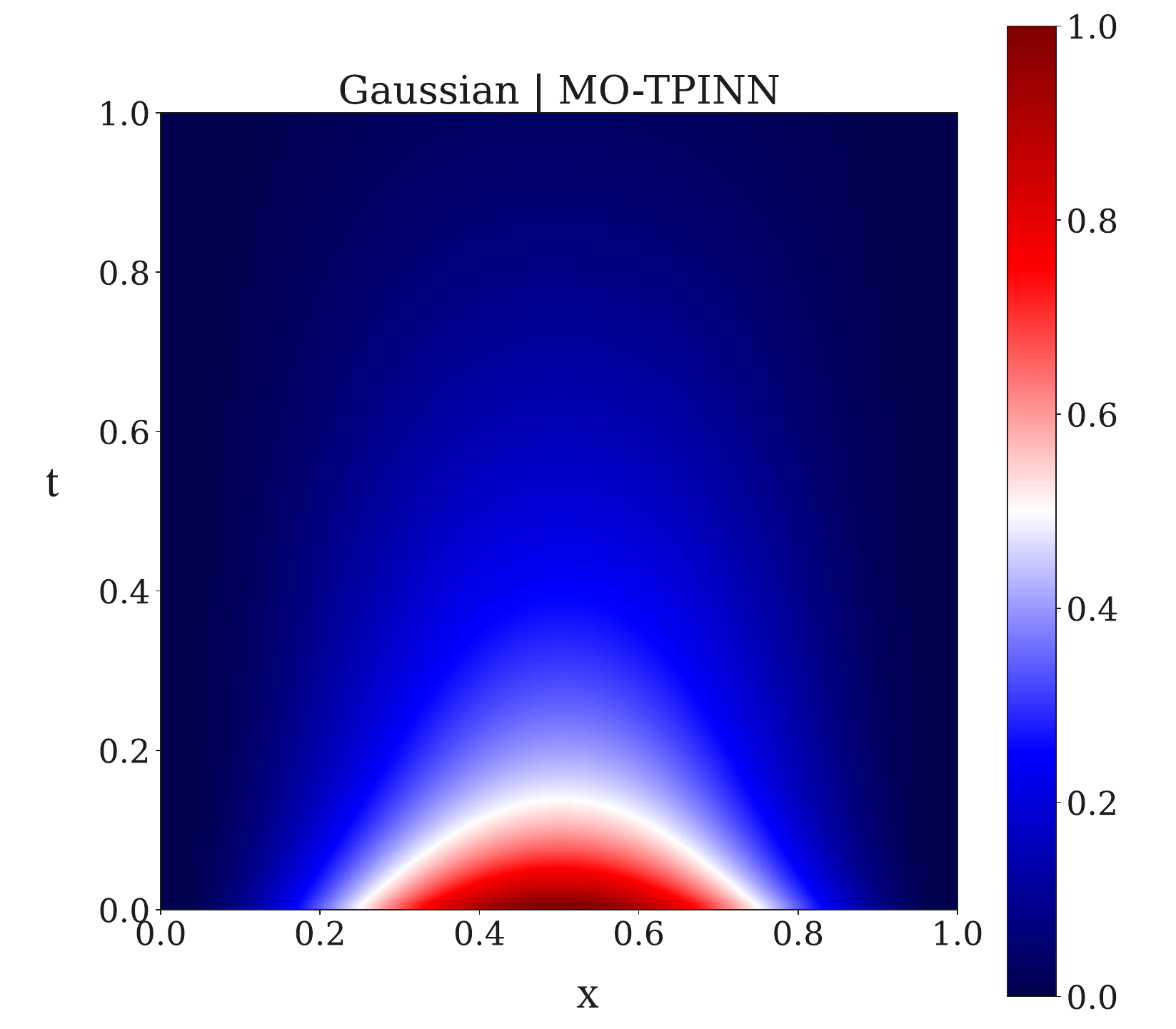}
\end{subfigure}
\caption{Heat 1D: Solution for the Gaussian kernel.}
\label{fig:Heat-Gaussiansol}       
\end{figure}

\begin{figure}[htb!]
\begin{subfigure}{.5\textwidth}
\center
\includegraphics[width=\textwidth]{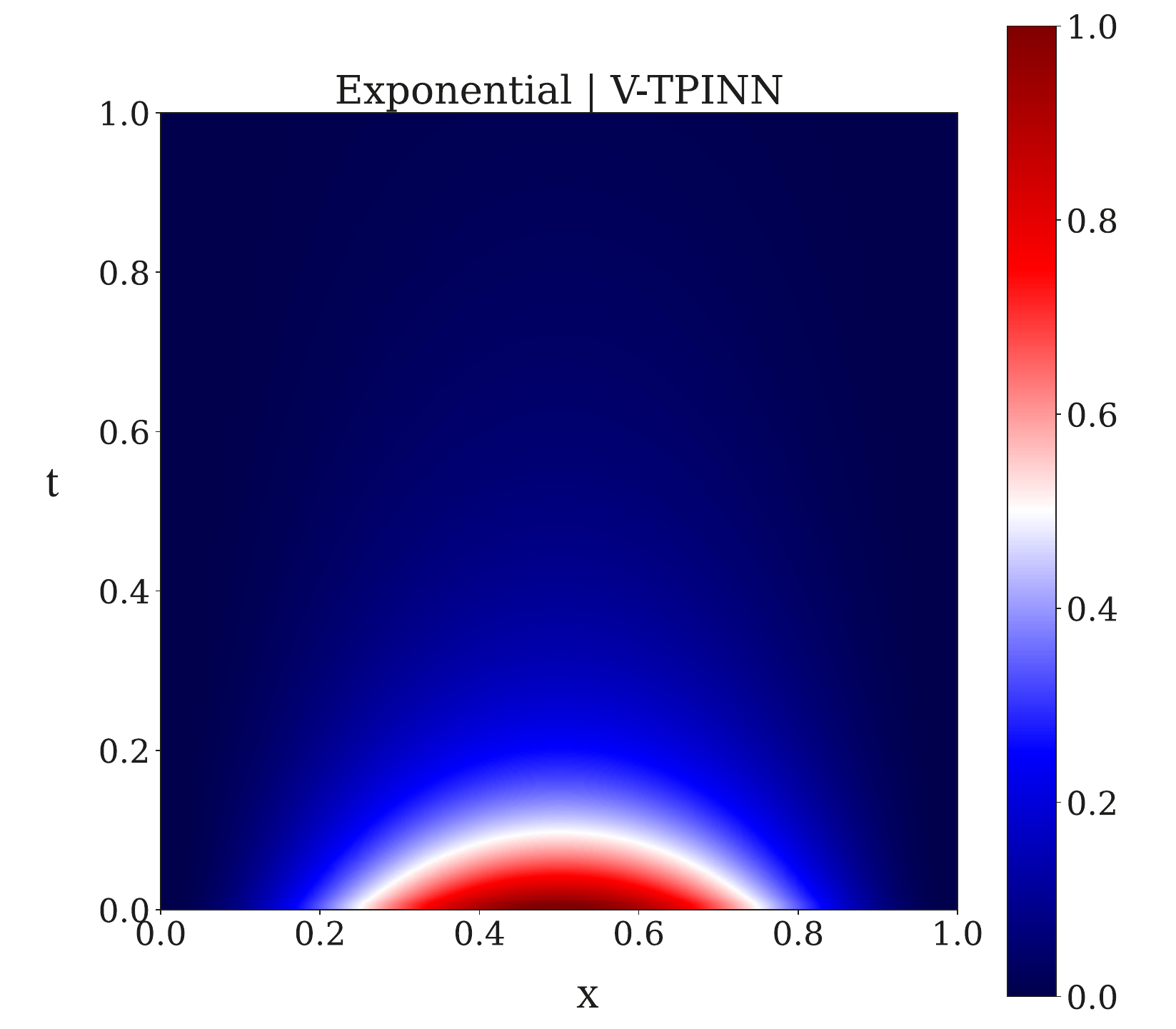}
\end{subfigure}
\begin{subfigure}{.5\textwidth}
\includegraphics[width=\textwidth]{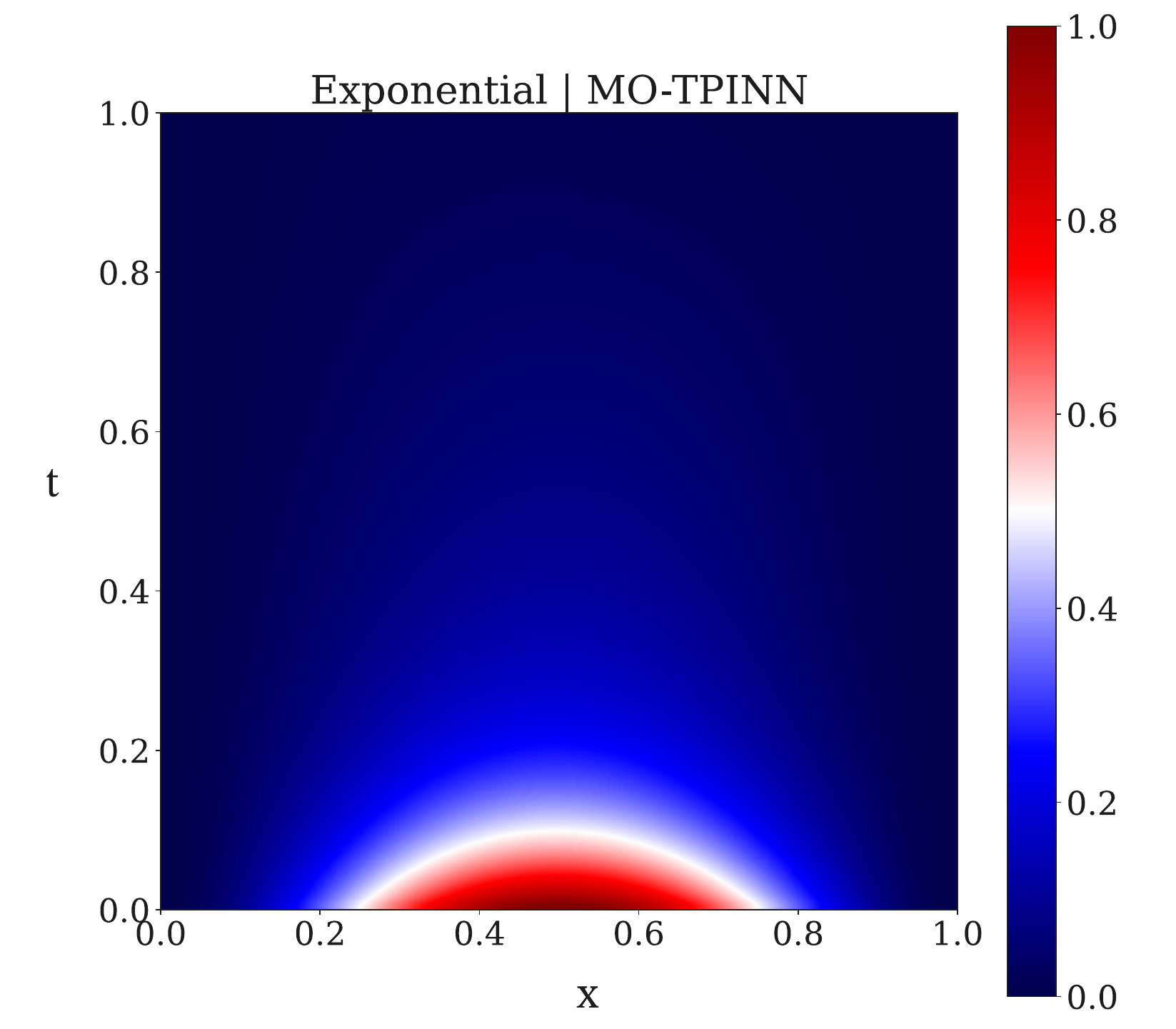}
\end{subfigure}
\caption{Heat 1D: Solution for the $1$-exponential kernel.}
\label{fig:Heat-Exponentialsol}       
\end{figure}

\section{Conclusion}\label{sec:Conclusion}
In this work, we introduced TPINNs to solve operator equations with stochastic right-hand side. We considered two architectures for TPINNs, namely V-TPINNs and MO-TPINNs, both showing efficient performance for \rev{across} a variety of operators. We included a bound for the generalization error for V--PINNs in \Cref{thm:boundGeneralization}, paving the way toward robust TPINNs and \rev{aiding in better understanding} their behavior. Our numerical experiments proved that TPINNs are a powerful technique, applicable to challenging problems. It is interesting to notice that none of them \rev{consistently outperformed} the other, justifying their introduction. \rev{Although a promising approach to overcoming the curse of dimensionality for tensor operator equations, V-TPINNs (resp. ML-TPINNs) have shown an increasing cost with $\tk$ due to automatic differentiation (resp.~difficulties inherent to training), opening research pathways towards more efficient schemes.} 

Further work includes coupling the results in \Cref{sec:Convergence} with shift theorems \cite{schwab2003sparseHigh} to supply complete convergence bounds of the behavior of PINNs. Also, TPINNs could be applied to UQ for small amplitude perturbed random domains \cite{sparse3,escapil2020helmholtz}. Alongside, hard BCs are no longer applicable for complex geometries, justifying the use of soft BCs TPINNs.

Finally, MO-TPINNs proved to be a surprisingly fast technique. In our opinion, \rev{improving training} is a promising research area. This could \rev{involve}: (i) optimizing weights during training, \rev{for example through} learning rate annealing \cite{gradientpathologies}; or (ii) \rev{conducting} hyperparameter optimization \cite{escapil2022hyper}. We also mention generalization bounds for MO-TPINNs. 

\bibliographystyle{siamplain}
\bibliography{references}

\begin{thebibliography}{10}

\bibitem{tensorflow2015-whitepaper}
{\sc M.~Abadi, A.~Agarwal, P.~Barham, E.~Brevdo, Z.~Chen, C.~Citro, G.~S.
  Corrado, A.~Davis, J.~Dean, M.~Devin, S.~Ghemawat, I.~Goodfellow, A.~Harp,
  G.~Irving, M.~Isard, Y.~Jia, R.~Jozefowicz, L.~Kaiser, M.~Kudlur,
  J.~Levenberg, D.~Man\'{e}, R.~Monga, S.~Moore, D.~Murray, C.~Olah,
  M.~Schuster, J.~Shlens, B.~Steiner, I.~Sutskever, K.~Talwar, P.~Tucker,
  V.~Vanhoucke, V.~Vasudevan, F.~Vi\'{e}gas, O.~Vinyals, P.~Warden,
  M.~Wattenberg, M.~Wicke, Y.~Yu, and X.~Zheng}, {\em {TensorFlow}:
  {L}arge-{S}cale {M}achine {L}earning on {H}eterogeneous {S}ystems}, 2015.
\newblock Software available from tensorflow.org.

\bibitem{barth2011multi}
{\sc A.~Barth, C.~Schwab, and N.~Zollinger}, {\em Multi-level {M}onte {C}arlo
  finite element method for elliptic {PDE}s with stochastic coefficients},
  Numerische Mathematik, 119 (2011), pp.~123--161.

\bibitem{bengio2017deep}
{\sc Y.~Bengio, I.~Goodfellow, and A.~Courville}, {\em Deep learning}, vol.~1,
  MIT press Cambridge, MA, USA, 2017.

\bibitem{bettencourt2019taylor}
{\sc J.~Bettencourt, M.~J. Johnson, and D.~Duvenaud}, {\em Taylor-mode
  automatic differentiation for higher-order derivatives in {JAX}},  (2019).

\bibitem{jax2018github}
{\sc J.~Bradbury, R.~Frostig, P.~Hawkins, M.~J. Johnson, C.~Leary,
  D.~Maclaurin, G.~Necula, A.~Paszke, J.~Vander{P}las, S.~Wanderman-{M}ilne,
  and Q.~Zhang}, {\em {JAX}: composable transformations of {P}ython+{N}um{P}y
  programs}, 2018.

\bibitem{Chen2020}
{\sc Y.~Chen, L.~Lu, G.~E. Karniadakis, and L.~D. Negro}, {\em Physics-informed
  neural networks for inverse problems in nano-optics and metamaterials}, Opt.
  Express, 28 (2020), pp.~11618--11633.

\bibitem{chernov2013first}
{\sc A.~Chernov and C.~Schwab}, {\em First order $k$-th moment finite element
  analysis of nonlinear operator equations with stochastic data}, Mathematics
  of Computation, 82 (2013), pp.~1859--1888.

\bibitem{HOQMC}
{\sc J.~Dick, Q.~T. Le~Gia, and C.~Schwab}, {\em Higher {O}rder {Q}uasi-{M}onte
  {C}arlo {I}ntegration for {H}olomorphic {P}arametric {O}perator {E}quations},
  SIAM/ASA Journal on Uncertainty Quantification, 4 (2016), pp.~48--79.

\bibitem{escapil2020helmholtz}
{\sc P.~Escapil-Inchausp{\'e} and C.~Jerez-Hanckes}, {\em Helmholtz scattering
  by random domains: first-order sparse boundary element approximation}, SIAM
  Journal on Scientific Computing, 42 (2020), pp.~A2561--A2592.

\bibitem{escapil2022hyper}
{\sc P.~Escapil-Inchauspé and G.~A. Ruz}, {\em Hyper-parameter tuning of
  physics-informed neural networks: {A}pplication to {H}elmholtz problems},
  Neurocomputing, 561 (2023), p.~126826.

\bibitem{multigroup}
{\sc C.~Fuenzalida, C.~Jerez-Hanckes, and R.~G. McClarren}, {\em Uncertainty
  {Q}uantification for {M}ultigroup {D}iffusion {E}quations {U}sing {S}parse
  {T}ensor {A}pproximations}, SIAM Journal on Scientific Computing, 41 (2019),
  pp.~B545--B575.

\bibitem{gerstner1998numerical}
{\sc T.~Gerstner and M.~Griebel}, {\em Numerical integration using sparse
  grids}, Numerical algorithms, 18 (1998), pp.~209--232.

\bibitem{gladstone2022fo}
{\sc R.~J. Gladstone, M.~A. Nabian, and H.~Meidani}, {\em {FO}-{PINN}s: A
  {F}irst-{O}rder formulation for {P}hysics {I}nformed {N}eural {N}etworks},
  arXiv preprint arXiv:2210.14320,  (2022).

\bibitem{sparse3}
{\sc H.~Harbrecht, R.~Schneider, and C.~Schwab}, {\em Sparse second moment
  analysis for elliptic problems in stochastic domains}, Numerische Mathematik,
  109 (2008), pp.~385--414.

\bibitem{kingma2014adam}
{\sc D.~P. Kingma and J.~Ba}, {\em Adam: {A} method for stochastic
  optimization}, arXiv preprint arXiv:1412.6980,  (2014).

\bibitem{lagaris1998artificial}
{\sc I.~E. Lagaris, A.~Likas, and D.~I. Fotiadis}, {\em Artificial neural
  networks for solving ordinary and partial differential equations}, IEEE
  transactions on neural networks, 9 (1998), pp.~987--1000.

\bibitem{lbfgs}
{\sc D.~C. Liu and J.~Nocedal}, {\em On the limited memory {BFGS} method for
  large scale optimization}, Mathematical Programming, 45 (1989), pp.~503--528.

\bibitem{lu2021deepxde}
{\sc L.~Lu, X.~Meng, Z.~Mao, and G.~E. Karniadakis}, {\em {DeepXDE}: {A} deep
  learning library for solving differential equations}, SIAM Review, 63 (2021),
  pp.~208--228.

\bibitem{luluhard2021}
{\sc L.~Lu, R.~Pestourie, W.~Yao, Z.~Wang, F.~Verdugo, and S.~G. Johnson}, {\em
  Physics-{I}nformed {N}eural {N}etworks with {H}ard constraints for {I}nverse
  {D}esign}, SIAM Journal on Scientific Computing, 43 (2021), pp.~B1105--B1132.

\bibitem{MIKAMatrixFreeKL}
{\sc M.~L. Mika, T.~J. Hughes, D.~Schillinger, P.~Wriggers, and R.~R.
  Hiemstra}, {\em A matrix-free isogeometric {G}alerkin method for
  {K}arhunen–{L}oève approximation of random fields using tensor product
  splines, tensor contraction and interpolation based quadrature}, Computer
  Methods in Applied Mechanics and Engineering, 379 (2021), p.~113730.

\bibitem{Mishra2021Inverse}
{\sc S.~Mishra and R.~Molinaro}, {\em {Estimates on the generalization error of
  physics-informed neural networks for approximating a class of inverse
  problems for PDEs}}, IMA Journal of Numerical Analysis, 42 (2021),
  pp.~981--1022.

\bibitem{Mishra2020EstimatesOT}
{\sc S.~Mishra and R.~Molinaro}, {\em {Estimates on the generalization error of
  physics-informed neural networks for approximating {PDE}s}}, IMA Journal of
  Numerical Analysis,  (2022).

\bibitem{omran2016some}
{\sc S.~Omran and A.~Sayed}, {\em On {S}ome {P}roperties of {T}ensor {P}roduct
  of {O}perators}, Global Journal of Pure and Applied Mathematics, 12 (2016),
  pp.~5139--5147.

\bibitem{fPINNs}
{\sc G.~Pang, L.~Lu, and G.~E. Karniadakis}, {\em f{PINN}s: {F}ractional
  {P}hysics-{I}nformed {N}eural {N}etworks}, SIAM Journal on Scientific
  Computing, 41 (2019), pp.~A2603--A2626.

\bibitem{perez2019stationary}
{\sc A.~P{\'e}rez-Obiol and T.~Cheon}, {\em Stationary real solutions of the
  nonlinear {S}chr{\"o}dinger equation on a ring with a defect}, Journal of the
  Physical Society of Japan, 88 (2019), p.~034005.

\bibitem{RAISSI2019686}
{\sc M.~Raissi, P.~Perdikaris, and G.~Karniadakis}, {\em Physics-informed
  neural networks: {A} deep learning framework for solving forward and inverse
  problems involving nonlinear partial differential equations}, Journal of
  Computational Physics, 378 (2019), pp.~686--707.

\bibitem{scarabosio2021deep}
{\sc L.~Scarabosio}, {\em Deep {N}eural {N}etwork {S}urrogates for {N}onsmooth
  {Q}uantities of {I}nterest in {S}hape {U}ncertainty {Q}uantification},
  SIAM/ASA Journal on Uncertainty Quantification, 10 (2022), pp.~975--1011.

\bibitem{schwab2003sparseHigh}
{\sc C.~Schwab and R.~A. Todor}, {\em Sparse finite elements for stochastic
  elliptic problems--higher order moments}, Computing, 71 (2003), pp.~43--63.

\bibitem{schwab2019deep}
{\sc C.~Schwab and J.~Zech}, {\em Deep learning in high dimension: {N}eural
  network expression rates for generalized polynomial chaos expansions in
  {UQ}}, Analysis and Applications, 17 (2019), pp.~19--55.

\bibitem{steinbach2007numerical}
{\sc O.~Steinbach}, {\em Numerical {A}pproximation {M}ethods for {E}lliptic
  {B}oundary {V}alue {P}roblems: {F}inite and {B}oundary {E}lements}, Texts in
  Applied Mathematics, Springer New York, 2007.

\bibitem{DeepUQ}
{\sc R.~K. Tripathy and I.~Bilionis}, {\em Deep {UQ}: Learning deep neural
  network surrogate models for high dimensional uncertainty quantification},
  Journal of Computational Physics, 375 (2018), pp.~565--588.

\bibitem{vonPetersdorff2006}
{\sc T.~von Petersdorff and C.~Schwab}, {\em Sparse finite element methods for
  operator equations with stochastic data}, Applications of Mathematics, 51
  (2006), pp.~145--180.

\bibitem{gradientpathologies}
{\sc S.~Wang, Y.~Teng, and P.~Perdikaris}, {\em Understanding and {M}itigating
  {G}radient {F}low {P}athologies in {P}hysics-{I}nformed {N}eural {N}etworks},
  SIAM Journal on Scientific Computing, 43 (2021), pp.~A3055--A3081.

\bibitem{yang2021b}
{\sc L.~Yang, X.~Meng, and G.~E. Karniadakis}, {\em B-{PINN}s: {B}ayesian
  physics-informed neural networks for forward and inverse {PDE} problems with
  noisy data}, Journal of Computational Physics, 425 (2021), p.~109913.

\bibitem{yang2022multi}
{\sc M.~Yang and J.~T. Foster}, {\em Multi-output physics-informed neural
  networks for forward and inverse {PDE} problems with uncertainties}, Computer
  Methods in Applied Mechanics and Engineering,  (2022), p.~115041.

\bibitem{yang2019adversarial}
{\sc Y.~Yang and P.~Perdikaris}, {\em Adversarial uncertainty quantification in
  physics-informed neural networks}, Journal of Computational Physics, 394
  (2019), pp.~136--152.

\bibitem{zhangSPDEPINNs}
{\sc D.~Zhang, L.~Guo, and G.~E. Karniadakis}, {\em Learning in {M}odal
  {S}pace: {S}olving {T}ime-{D}ependent {S}tochastic {PDE}s {U}sing
  {P}hysics-{I}nformed {N}eural {N}etworks}, SIAM Journal on Scientific
  Computing, 42 (2020), pp.~A639--A665.

\bibitem{ZHANG2019108850}
{\sc D.~Zhang, L.~Lu, L.~Guo, and G.~E. Karniadakis}, {\em Quantifying total
  uncertainty in physics-informed neural networks for solving forward and
  inverse stochastic problems}, Journal of Computational Physics, 397 (2019),
  p.~108850.

\bibitem{zhu2021local}
{\sc Q.~Zhu and J.~Yang}, {\em A local deep learning method for solving high
  order partial differential equations}, arXiv preprint arXiv:2103.08915,
  (2021).

\bibitem{ZHU201956}
{\sc Y.~Zhu, N.~Zabaras, P.-S. Koutsourelakis, and P.~Perdikaris}, {\em
  Physics-constrained deep learning for high-dimensional surrogate modeling and
  uncertainty quantification without labeled data}, Journal of Computational
  Physics, 394 (2019), pp.~56--81.

\bibitem{zou2022neuraluq}
{\sc Z.~Zou, X.~Meng, A.~F. Psaros, and G.~E. Karniadakis}, {\em Neural{UQ}:
  {A} comprehensive library for uncertainty quantification in neural
  differential equations and operators}, arXiv preprint arXiv:2208.11866,
  (2022).

\end{thebibliography}
\end{document}